\documentclass[twoside]{article}
\PassOptionsToPackage{numbers, compress}{natbib}

\usepackage[T1]{fontenc}    
\usepackage{hyperref}       
\usepackage{url}            
\usepackage{booktabs}       
\usepackage{amsfonts}       
\usepackage{nicefrac}       
\usepackage{microtype}      

\usepackage{microtype}
\usepackage{graphicx}
\usepackage{subfigure}
\usepackage{booktabs} 
\usepackage{amsfonts}
\usepackage{amsmath} 

\usepackage{hyperref}

\usepackage{amsmath,amssymb,amsthm}
\usepackage{booktabs} 
\usepackage{amscd}
\usepackage{mathrsfs} 
\usepackage[shortlabels]{enumitem}
\usepackage{euscript}
\usepackage{graphicx}
\usepackage{amscd,bm}
\usepackage{caption}

\usepackage{calrsfs}

\DeclareMathAlphabet{\pazocal}{OMS}{zplm}{m}{n}
\DeclareMathAlphabet\mathbfcal{OMS}{cmsy}{b}{n}

\usepackage{algorithm}
\usepackage{algorithmic}
\newtheorem{theorem}{Theorem}
\newtheorem{definition}{Definition}
\newtheorem{lemma}{Lemma}
\newtheorem{proposition}{Proposition}
\newtheorem{remark}{Remark}

\newtheorem{corollary}{Corollary}

\providecommand{\nor}[1]{\ensuremath{\left\lVert {#1} \right\rVert}}

\providecommand{\scal}[2]{\ensuremath{\left\langle{#1},{#2}\right\rangle}}

\providecommand{\scalT}[2]{\ensuremath{\left\langle{#1},{#2}\right\rangle}}

\newcommand{\R}{\mathbb R}
\providecommand{\scal}[2]{\ensuremath{\left\langle{#1},{#2}\right\rangle}}

\def\argmax{\operatornamewithlimits{arg\,max}}

\def\bit{\begin{itemize}}
\def\eit{\end{itemize}}

\def\ben{\begin{enumerate}}
\def\een{\end{enumerate}}

\usepackage{listings}
\usepackage{color}

\definecolor{dkgreen}{rgb}{0,0.6,0}
\definecolor{gray}{rgb}{0.5,0.5,0.5}
\definecolor{mauve}{rgb}{0.58,0,0.82}

\usepackage{enumitem}

\usepackage{natbib}

\usepackage[accepted]{aistats2019}
\begin{document}

%

%
\runningauthor{Youssef Mroueh, Tom Sercu,  Anant Raj}

\twocolumn[
\aistatstitle{Sobolev Descent}
\aistatsauthor{ Youssef Mroueh$^{\dagger}$ \And Tom Sercu$^{\dagger}$ \And  Anant Raj$^{\star}$ }
\aistatsaddress{ $\dagger$ IBM Research, MIT-IBM Watson Lab $\star$  MPI  } ]

\begin{abstract}
We study a simplification of GAN training: the problem of transporting particles from a source to a target distribution.
Starting from the Sobolev GAN critic, part of the gradient  regularized GAN family, we show a strong relation with Optimal Transport (OT). Specifically with the less popular \emph{dynamic} formulation of OT that finds a path of distributions from source to target minimizing a ``kinetic energy''.
We introduce Sobolev descent that constructs similar paths by following  
gradient flows of a critic function in a kernel space or parametrized  by a neural network.
In the kernel version, we show convergence to the target distribution in the MMD sense.
We show in theory and experiments that regularization has an important role
in favoring smooth transitions between distributions,
avoiding large gradients from the critic.
This analysis in a simplified particle setting provides insight in paths to equilibrium in GANs.
\end{abstract}

\section{Introduction}

We  study the problem of transporting particles (cloud of high dimensional points) from a source to a target distribution, by incrementally following  gradient flows  of a critic function (Sobolev critic). We call this incremental process Sobolev Descent.
This can be seen as a simplified version of GAN training dynamics: the generator is replaced by a  set of $N$ particles in $\R^d$. The particles define a time evolving distribution $\nu_{q_{t}}$.
Rather than min-max optimization in GANs, we only have maximization of the critic function $f$ at each timestep $t$.
We parametrize the critic  either in an RKHS or with neural networks, leading us to Regularized Kernel and Neural Sobolev Descent respectively.

Optimal Transport (OT) \cite{Villani,santambrogio2015optimal,peyre2017computational} is increasingly gaining interest in the machine learning community.
The static formulation of OT, seeks an  optimal bijection $T$, defining a push forward operator from $q$ to $p$: $T_{\#}\nu_q=\nu_p$ (i.e. Monge problem, relaxed by Kantorovic to seek a coupling $\pi$ rather than a bijection $T$). 
While this static viewpoint is the most popular (e.g. WGAN \cite{WGAN} or recently \cite{salimans2018improving,geneway2017learning}),
we will be focusing instead on the \emph{dynamic} formulation of the Wasserstein-2 distance,
for which Benamou and Brenier \cite{dynamicTransport} showed that the OT problem has a fluid dynamic interpretation:
\begin{eqnarray}
& {W}^2_2(\nu_p,\nu_q)= \inf_{q_t,V_t} \int_{0}^1 \int \nor{V_t(x)}^2d\nu_{q_{t}}(x) dt\nonumber\\
& \text{ s.t } \frac{\partial q_t(x)}{\partial t}=-div(q_t V_t(x))~~ q_0=q, q_1=p.
 \label{eq:Benamou}
 \end{eqnarray}
The optimal transport problem in this perspective corresponds to finding a \emph{path of densities} $q_t$ advecting from $q$ to $p$ with optimal \emph{velocity fields} $V_{t}$ that minimize the \emph{kinetic energy}.
Note that a major limitation is the need for an explicit analytic expression of $p$ and $q$ in order to solve  for $q_t, V_t$ in Eq. \ref{eq:Benamou}.
\begin{figure}[ht!]
\centering
\includegraphics[width=0.45\textwidth]{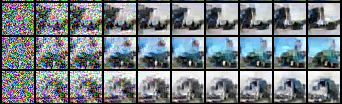}
\caption{Neural Sobolev descent paths in the space of $32\times32$ images. The source distribution here is as in GAN uniform noise and the target is the truck class in CIFAR 10. The main difference with GAN is that images in Sobolev descent are the particles moving along the Sobolev critic, while in GAN the generator adapts in a min-max game with the critic. }
\vskip -0.2in 
\label{fig:trucksintro}
\end{figure}

In GANs, the generator is updated with stochastic gradient descent along directions of the discriminator (=critic) gradient $\nabla_x f(x)$, which immediately suggests a link with the velocity fields $V_t$ in dynamic OT. In the recent GAN literature, variations have been studied where the gradient of the critic ($\mathbb{E}_{x \sim\mu} [ \nor{\nabla f(x)}]$) is constrained by adding a gradient penalty in the objective \cite{SobolevGAN,gulrajani2017improved}.
We will show that for this specific class of critics, in the simplified particle descent setting,
we construct paths between source and target distributions that
minimize a form of \emph{kinetic energy}.
Two advantages set it apart from dynamic OT: 
1) we need only samples from $p$ and $q$, and
2) the method is scalable (in sample size, input dimension and time complexity) because $f(x)$ is parametrized in an RKHS or with a neural network.

To define Sobolev  descent we start from the recently introduced regularized Kernel Sobolev Discrepancy \cite{mroueh2018kernelsobdiscrepTechRep} as a way to quantify \emph{the kinetic energy} that we wish to minimize (Section \ref{sec:KE}). 
We construct in Section \ref{sec:SDescent} paths of distributions from source to target that minimize this notion of kinetic energy.
We prove that under mild assumption kernel Sobolev descent converges in the MMD (Maximum Mean Discrepancy \cite{mmd}) sense: $\text{MMD}(\nu_p,\nu_{q_t})\to 0$ as $t\to \infty$.
We highlight the prominent role of regularization in getting tunable smooth paths
which relates to stable training in the GAN setting.
We discuss the connections to dynamic OT \cite{dynamicTransport} and Stein Descent of \cite{Stein,steindescent} in Section \ref{sec:previouswork}. Finally in  Section \ref{sec:alg}, we give  algorithms for Kernel and Neural Sobolev Descent.
We show the validity of our approach on synthetic data, image coloring and shape morphing and compare to classic OT algorithms.
We then validate that Sobolev descent is a proxy for GANs on high dimensional data: we move particles $\in \R^{3\times32\times32}$ from noise to match CIFAR10 images (Figure \ref{fig:trucksintro}).

\section{Kernel Sobolev Discrepancy}\label{sec:KE}

In this Section we review the Kernel Sobolev Discrepancy recently introduced in~\cite{mroueh2018kernelsobdiscrepTechRep}.
The Kernel Sobolev Discrepancy will be fundamental to our particles descent as it defines the notion of kinetic energy to be minimized.

\textbf{Sobolev Discrepancy.}
The Sobolev discrepancy was introduced recently in the context of Generative Adversarial Networks in Sobolev GAN \cite{SobolevGAN}.
We start by defining the Sobolev Discrepancy. Let $\pazocal{X}$ be a compact space in $\mathbb{R}^d$ with lipchitz boundary $\partial \pazocal{X}$.
\begin{definition}[Sobolev Discrepancy \cite{SobolevGAN,mroueh2018kernelsobdiscrepTechRep}] Let $\nu_{p},\nu_{q}$ be two measures defined on $\pazocal{X}$. The Sobolev Discrepancy is defined as follows:
\begin{align*}
\pazocal{S}(\nu_{p},\nu_{q})&= \sup_{f} \Big\{ \underset{x\sim \nu_{p}}{\mathbb{E}}f(x)-\underset{x\sim \nu_{q}}{\mathbb{E}}f(x)  
\Big\}\\
&\text{s.t } f \in W^{1,2}_0(\pazocal{X},\nu_{q}), \underset{x\sim \nu_{q}}{\mathbb{E}}\nor{\nabla_xf(x) }^2 \leq 1\\
&=\inf_{f}\Big \{\sqrt{\int_{\pazocal{X}} \nor{\nabla_x f(x)}^2 d\nu_{q}(x)} \Big\}\\
&\text{s.t }  p(x)-q(x)=-div(q(x)\nabla_x f(x)), f|_{\partial \pazocal{X}}=0
\label{eq:Sobolev}
\end{align*}
and $W^{1,2}_0(\pazocal{X},\nu_{q})= \{ f \text{ vanishes at the boundary of } \pazocal{X}\\ \text{ and }\underset{x\sim \nu_{q}}{\mathbb{E}} \nor{\nabla_x f(x)}^2< \infty\}$.
\end{definition}

We refer to $\nu_p$ as the target distribution, and $\nu_q$ as the source distribution.
The Sobolev discrepancy finds a witness function (or critic) that maximizes the mean discrepancy between the source and target distribution,
while constraining the witness function gradients semi-norm to be in a weighted Sobolev ball (under the source distribution $\nu_q$).
Note that the $\sup$ form (dual) is computationally friendly since it can be optimized using samples from $p$ and $q$.
The $\inf$ form (primal) sheds light on the physical meaning of this discrepancy:
it is the \textbf{minimum kinetic energy} needed to advect the mass $q$ to $p$ following gradients of a critic.
This interpretation will play a crucial role in Sobolev Descent. 

\textbf{Regularized Kernel Sobolev Discrepancy (RKSD).}
In order to define Sobolev descent we need to introduce a last ingredient: the Kernelized Sobolev Discrepancy.
In other words a kernelized measure of \emph{minimum kinetic energy for transporting $q$ to $p$}. To simplify the presentation we give in the main paper results for finite dimensional RKHS, all results for infinite dimensional RKHS are given in  Appendix \ref{app:InfDim}. 
 
\textbf{RKHS Properties and Assumptions.} Let $\mathcal{H}$ be a \emph{finite dimensional RKHS} with a finite feature map
$\Phi: x\to \Phi(x)\in \mathbb{R}^m$, hence with Kernel $k,$ $k(x,y)=\scalT{\Phi(x)}{\Phi(y)}=\sum_{j=1}^m \Phi_j(x)\Phi_{j}(y)$,
where $\scalT{}{}$ is the dot product in $\mathbb{R}^m$. Note that for a function $f\in \mathcal{H}$, $f(x)=\scalT{\boldsymbol{f}}{\Phi(x)}$, where $\boldsymbol{f}\in \mathbb{R}^m$
and $\nor{f}_{\mathcal{H}}=\nor{\boldsymbol{f}}$. Let $J\Phi(x) \in \mathbb{R}^{d\times m}$ be the jacobian of $\Phi$, $[J\Phi]_{a,j}(x)=\frac{\partial}{\partial x_a}\Phi_j(x) $.
We have the following expression of the gradient  $\nabla_x f(x)=(J\Phi(x)\boldsymbol{f}) \in \mathbb{R}^d$. 
Mild assumptions on $\mathcal{H}$ are required ($\Phi$ bounded and differentiable (A1), has bounded derivatives (A2), and zero boundary condition on $\Phi$ (A3)) and can be found in \cite{mroueh2018kernelsobdiscrepTechRep}.
\begin{remark}
Assumption (A3) on zero boundary condition can be weakened to  $q(x)\scalT{\nabla_x u^{\lambda}_{p,q}(x)}{n(x)}=0$ on $\partial \pazocal{X}$ ($n(x)$ is the normal on $\partial \pazocal{X}$). Assuming $\pazocal{X}=\mathbb{R}^d$ and that $q$ and $p$ vanish at $\infty$ we can use non vanishing feature maps $\Phi$ on $\partial \pazocal{X}$.
\end{remark}

The Kernel Sobolev Discrepancy \cite{mroueh2018kernelsobdiscrepTechRep} restricts the witness function of the Sobolev discrepancy to a finite dimensional RKHS $\mathcal{H}$,
with feature map $\Phi$. The Regularized Kernel Sobolev Discrepancy further regularizes the critic using Tikhonov regularization.
\vskip -0.1in
\begin{definition}[RKSD] Let $\mathcal{H}$ be a finite dimensional RKHS satisfying assumptions A1, A2 and A3.
Let $\lambda>0$ be the regularization parameter.
Let $\nu_{p},\nu_{q}$ be two measures defined on $\pazocal{X}$. The regularized Kernel Sobolev discrepancy restricted to the space $\mathcal{H}$ is defined as follows:
\begin{eqnarray}
\pazocal{S}_{\mathcal{H},\lambda}(\nu_p,\nu_q)&=& \sup_{f \in\mathcal{H} }\Big \{ \underset{x\sim \nu_{p}}{\mathbb{E}}f(x)-\underset{x\sim \nu_{q}}{\mathbb{E}}f(x) \}\nonumber\\
&\text{s.t }&\underset{x\sim \nu_{q}}{\mathbb{E}}\nor{\nabla_xf(x) }^2+\lambda \nor{f}^2_{\mathcal{H}} \leq 1\nonumber\\
\label{eq:SobolevH}
\end{eqnarray}
\end{definition}
\vskip -0.1in
We identify in the constraint in Equation \eqref{eq:SobolevH} a regularized operator defined by
 \begin{equation}
 D(\nu_q)=\mathbb{E}_{x\sim\nu_q}([J\Phi(x)]^{\top}J\Phi(x)).
 \label{eq:D}
 \end{equation}
 The constraint can be written as $\scalT{\boldsymbol{f}}{(D(\nu_q)+\lambda I_{m})\boldsymbol{f}}\leq 1$.
Following \cite{mroueh2018kernelsobdiscrepTechRep} we call $D(\nu_q)$ the Kernel Derivative Gramian Embedding (KDGE) of $\nu_q$.
KDGE is an operator embedding of the distribution.The KDGE can be seen as ``covariance'' of the jacobian. 
This operator embedding of $\nu_q$ is to be contrasted with the classic Kernel Mean Embedding (KME) of a distribution, $$\boldsymbol{\mu}(\nu_q)={\mathbb{E}}_{x\sim \nu_{q}}\Phi(x).$$
The KDGE can be thought as covariance of velocity fields (more on this intuition in Section \ref{sec:principal} ). 

The following proposition proved in \cite{mroueh2018kernelsobdiscrepTechRep} summarizes properties of the squared RKSD :
\begin{proposition} [Closed Form Expression of RKSD] 
Let $\lambda>0$. We have:
$\pazocal{S}^2_{\mathcal{H},\lambda}(\nu_p,\nu_q)= \sup_{\boldsymbol{u} \in \mathbb{R}^m } 2\scalT{\boldsymbol{u}}{\boldsymbol{\mu}(\nu_p)-\boldsymbol{\mu}(\nu_q)}-\scalT{\boldsymbol{u}}{(D(\nu_q)+\lambda I_m) \boldsymbol{u}}$.
This has the following closed form:$$\pazocal{S}^2_{\mathcal{H},\lambda}(\nu_p,\nu_q)= \nor{(D(\nu_q)+\lambda I_m)^{-\frac{1}{2}}\left(\boldsymbol{\mu}(\nu_p)-\boldsymbol{\mu}(\nu_q)\right)}^2$$
and the optimal witness function $u^{\lambda}_{p,q}$ of $\pazocal{S}^2_{\mathcal{H},\lambda}(\nu_p,\nu_q)$ satisfies: $u^{\lambda}_{p,q}(x)=\scalT{\boldsymbol{u}^{\lambda}_{p,q}}{\Phi(x)}$ where
$$(D(\nu_q)+\lambda I_m)\boldsymbol{u}^{\lambda}_{p,q}=\boldsymbol{\mu}(\nu_p)-\boldsymbol{\mu}(\nu_q).$$
Note that $\pazocal{S}^2_{\mathcal{H},\lambda}(\nu_p,\nu_q)=\int_{\pazocal{X}} \nor{\nabla_x u^{\lambda}_{p,q}(x)}^2 q(x)dx+\lambda \nor{\boldsymbol{u}^{\lambda}_{p,q}}^2 $
is the minimum regularized kinetic energy for advecting $q$ to $p$ using gradients of potentials in $\mathcal{H}$.
\end{proposition}
Note that RKSD is related to one of the most commonly used distances between distributions via embedding in RKHS, the maximum mean discrepancy \cite{mmd}
\[\text{MMD}(\nu_p,\nu_{q_t})=\nor{\boldsymbol{\mu}(\mu_p)-\boldsymbol{\mu}(\nu_{q_{t}}},\]
with the main difference is that the KMEs in RKSD are whitened in the space defined by the KDGE defined in Eq. \eqref{eq:D}.
\begin{remark}
  a) From this proposition we see that $\nabla_xu^{\lambda}_{p,q}(x)$ can be seen as velocities of minimum regularized kinetic energy, advecting $q$ to $p$.
  b) We give here the expression of the witness function $u^{\lambda}_{p,q}$ of $\pazocal{S}^2_{\mathcal{H},\lambda}(\nu_p,\nu_q)$ rather than $\pazocal{S}_{\mathcal{H},\lambda}(\nu_p,\nu_q)$ for convenience.
  The witness in \eqref{eq:SobolevH} is 
$u^{\lambda}_{p,q}/\pazocal{S}_{\mathcal{H},\lambda}$.

\end{remark}

\textbf{Empirical RKSD.}
An estimate of the Sobolev critic given finite samples from $p$ and $q$ $\{x_i,i=1\dots N, x_i \sim p\}$, and $\{y_i,i=1\dots M, y_i \sim q\}$ is straightforward:
$\hat{u}^{\lambda}_{p,q}(x)=\scalT{\boldsymbol{\hat{u}}^{\lambda}_{p,q}}{\Phi(x)}_{\mathbb{R}^m},$
where $\boldsymbol{\hat{u}}^{\lambda}_{p,q}=(\hat{D}(\hat{\nu}_q)+\lambda I_m)^{-1} \left(\hat{\boldsymbol{\mu}}(\hat{\nu}_p)-\hat{\boldsymbol{\mu}}(\hat{\nu}_q) \right)$.
With the empirical KDGE is given by $\hat{D}(\hat{\nu_q})=\frac{1}{M}\sum_{j=1}^M [J\Phi(y_j)]^{\top}J\Phi(y_j)$, and the emprical KMEs $\hat{\boldsymbol{\mu}}(\hat{\nu}_p)=\frac{1}{N}\sum_{i=1}^N \Phi(x_i)$ and $\hat{\boldsymbol{\mu}}(\hat{\nu}_q)=\frac{1}{M}\sum_{j=1}^M \Phi(y_j)$.

\section{Sobolev Descent}\label{sec:SDescent}

\textbf{Discrete Sobolev Descent.} Now that we have a notion of Kernelized kinetic energy (the RKSD) and velocity fields consisting of the gradients of the Sobolev critic that achieve the
minimum kinetic energy, we are ready to introduce the Sobolev Descent.
Our main result will be to construct an infinitesimal transport map $T^\varepsilon$ of the source distribution $\nu_q$, and show that the resulting distribution $\nu_{q_{[T^\varepsilon]}}$ converges to the target distribution $\nu_p$ in the MMD sense.
For $x \sim \nu_q$, moving along the gradient flow of the optimal regularized Sobolev critic $u^{\lambda}_{p,q}$ results in a decrease in MMD.
We prove in Theorem \ref{theo:PertubationMMDSobolev} (all proofs are given in Appendix \ref{app:Proofs}) that, using the infinitesimal transport map:
$$T^\varepsilon(x)=x+\varepsilon \nabla_x u^{\lambda}_{p,q}(x),~ x\sim \nu_q,$$
the push forward $T^{\varepsilon}_{\#}\nu_q=\nu_{q_{[T^\varepsilon]}}$ ensures that this transport map decreases the  MMD in the following sense:
$$\frac{d}{d\varepsilon}\text{MMD}^2(\nu_p, T^{\varepsilon}_{\#}\nu_q)\Big|_{\varepsilon=0} \leq 0,$$
where the first variation $\frac{d}{d\varepsilon}\text{MMD}^2(\nu_p, T^{\varepsilon}_{\#}\nu_q)|_{\varepsilon=0}= \lim_{\varepsilon \to 0} \frac{\text{MMD}^2(\nu_p, T^{\varepsilon}_{\#}\nu_q)-\text{MMD}^2(\nu_p,\nu_q) }{\varepsilon}$.

\begin{theorem}[Gradient flows of the Regularized Sobolev Critic decrease the MMD distance]Let $\lambda>0$.
Let $u^{\lambda}_{p,q}$ be the solution of the regularized Kernel Sobolev discrepancy between $\nu_p$ and $\nu_q$ i.e $\boldsymbol{u}^{\lambda}_{p,q}=(D(\nu_q)+\lambda I_m)^{-1}(\boldsymbol{\mu}(\nu_p)-\boldsymbol{\mu}(\nu_q)).$
Consider the infinitesimal transport of $\nu_q$ via $T^\varepsilon(x)=x+\varepsilon \nabla_x u^{\lambda}_{p,q}(x)$. 
We have the following first variation of the $\text{MMD}^2$ under this particular perturbation:
\begin{align*}
&\frac{d}{d\varepsilon}\text{MMD}^2(\nu_p, T^{\varepsilon}_{\#}\nu_q)\Big|_{\varepsilon=0}\\
&=- 2\left(\text{MMD}^2(\nu_p,\nu_q)-\lambda \pazocal{S}^2_{\mathcal{H},\lambda}(\nu_p,\nu_q)\right) \leq 0. 
\end{align*}
\label{theo:PertubationMMDSobolev}
\vskip -0.2 in
\end{theorem}
\vskip -0.2 in
\begin{remark} 1)  The $\leq 0$ of the RHS above is guaranteed since for any $\lambda>0$ we have
  $\pazocal{S}^2_{\mathcal{H},\lambda}(\nu_p,\nu_q)\leq\nor{(D(\nu_q)+\lambda I)^{-1/2}}^2_{op}\nor{\boldsymbol{\mu}(\nu_p)-\boldsymbol{\mu}(\nu_q)}^2\leq\frac{1}{\lambda} \text{MMD}^2(\nu_p,\nu_q)$ (where $\nor{.}_{op}$ is the operator norm).
  2) Assume $D(\nu_q)$ is non singular, Theorem \ref{theo:PertubationMMDSobolev} holds true for $\lambda=0$.
\end{remark}

From this Theorem we see that  when we move the mass from $q$ to $p$ along  the gradient flows of the regularized Sobolev critic, this results in a decrease in the MMD. Hence we  are making  progress towards matching $p$ in the MMD sense. The amount of progress is proportional to $(\text{MMD}^2(\nu_p,\nu_q)-\lambda \pazocal{S}^2_{\mathcal{H},\lambda}(\nu_p,\nu_q)) := \Delta_q$.

Theorem \ref{theo:PertubationMMDSobolev} suggests an iterative procedure that transports a source distribution $\nu_q$ to a target distribution $\nu_p$: we start with applying transform $T^\varepsilon_0(x)=x+\varepsilon \nabla_x u^{\lambda}_{p,q_0}(x)$
on $q_0=q$ which decreases the squared MMD distance by $\Delta_{q_0}$. This results in a new distribution $q_1(x) = q_{0 [T^\varepsilon_0]}(x)$.
To further decrease the MMD distance we apply a new transform on $q_1$, $T^\varepsilon_1(x)= x+\varepsilon \nabla_x u^{\lambda}_{p,q_1}(x)$; this results in a decrease of the squared MMD distance by $\Delta_{q_1}$.
By iterating this process we construct a path of distributions $\{q_{\ell}\}_{\ell=0\dots L-1}$ between $q_0$ and $p$: 
\begin{equation}
q_{\ell+1}=q_{\ell,[T^\varepsilon_{\ell}]} ~ \text{ where } T^\varepsilon_{\ell}(x)= x+\varepsilon \nabla_x u^{\lambda}_{p,q_{\ell}}(x), x \sim \nu_{q_{\ell}}.
\label{eq:path}
\end{equation}
We call this iterative process Sobolev Descent, and this incremental decrease in the MMD distance is summarized in the following corollary:

\begin{corollary}[Regularized Sobolev Descent Decreases the MMD]
Consider the path of distributions $q_{\ell}$ between $q_0=q$ and $p$ constructed in equation \eqref{eq:path} we have for $\ell \in \{0,\dots L-1\}$:
$\frac{d}{d\varepsilon}\text{MMD}^2(\nu_p,\nu_{q_{\ell+1}})\Big|_{\varepsilon=0}=- 2\left(\text{MMD}^2(\nu_p,\nu_{q_{\ell}})-\lambda \pazocal{S}^2_{\mathcal{H},\lambda}(\nu_p,\nu_{q_{\ell}})\right)\leq 0. $
\label{corr:descent}
\end{corollary}

\paragraph{Continuous Sobolev Descent.} We have refered to points advecting from $q$ to $p$ via Sobolev descent as particles.
Let $t=\ell \varepsilon$ be the time variable, hence the time stepsize $dt=\varepsilon$ . Note $\nu_{q_t}$ the measure of the moving particles $X_{t}$  at time $t$. The continuous Sobolev descent  can be defined at the limit $\varepsilon \to 0$  as  the following non linear advection process on particles $X_{t}$ (whose distribution is $\nu_{q_t}$) advecting from $q$ to $p$ following the flow of the Sobolev critic :
$$dX_{t}= \nabla_{x}u_{p,q_{t}}(x) dt, X_0 \sim \nu_{q},$$
where $u_{p,q_{t}}$ is the Sobolev witness function between $\nu_p$ and $\nu_{q_{t}}$. In the next section we will analyse the convergence of the continuous Sobolev descent to the target distribution $\nu_p$.
  
\begin{figure*}[ht!]
\centering
\includegraphics[height=80px]{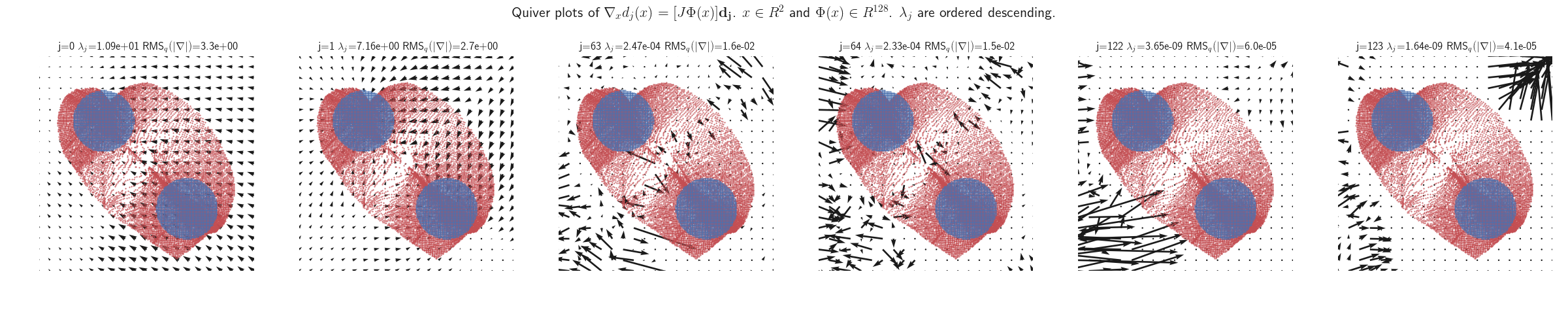}
\includegraphics[height=80px]{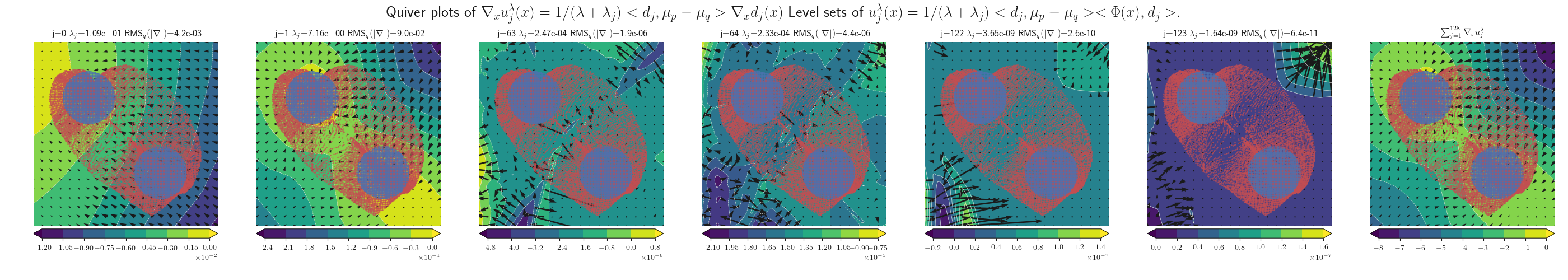}
\vskip -0.2in
\caption{
The principal transport directions for an intermediate state $q_t$ (red cloud) in the shape morphing application with Neural Sobolev Descent (see Figure \ref{fig:morphing_nsd}).
The top row shows $\nabla_x d_j(x)$, bottom row shows $\nabla_x u_j^\lambda(x)$ for $\lambda=0.3$.
Note how small $j$ (large eigenvalues) correspond to smooth vectorfields where the vectors have large norm (as measured in RMS over the points in point cloud $x \sim \nu_{q_t}$).
The intermediate and large $j$ values correspond to non-smooth vectorfields and non-smooth motions. For $\nabla_x u_j^\lambda(x)$, the principal transport directions $\nabla_x d_j(x)$ are multiplied with $\frac{1}{\lambda+\lambda_j}$ and the inner product with $\mu_p-\mu_q$ (two scalar multipliers). We see the non-smooth $\nabla_x u_j^\lambda(x)$ (small $\lambda_j$) have small RMS norm and contribute less, as they are effectively filtered out by the smoothing parameter $\lambda$.
The bottom right subplot shows the total critic $u^\lambda(x)=\sum_{j=1}^m u^\lambda_j(x)$.
}
\label{fig:Transp}
\vskip -0.2in
\end{figure*}

\subsection{Convergence of  Continuous Sobolev Descent}\label{sec:convergence}


In order to analyze the convergence of the continuous Sobolev descent, we will formulate the progress of MMD as a differential equation in time and show that the right hand side is always negative.
There are two distinct cases.

\textbf{Case 1: $\lambda=0$, Unregularized Sobolev Discrepancy Flows.}
Assume that $D(\nu_{q_{t}})$ is non singular, for all time steps $t$.
 Corollary \ref{corr:descent} suggests the following dynamic of the MMD for the continuous descent:
$$\frac{d}{dt}\text{MMD}^2(\nu_p,\nu_{q_t})=-2\text{MMD}^2(\nu_p,\nu_{q_t}).$$
This suggests a fast exponential convergence of $\nu_{q_{t}}$ to $\nu_p$ in the MMD sense: $\text{MMD}^2(\nu_p,\nu_{q_t})=e^{-2t}\text{MMD}^2(\nu_p,\nu_{q})$, i.e $\text{MMD}^2(\nu_p,\nu_{q_t})\to 0,$ as $t\to \infty$.
This fast convergence is not necessarily desirable as it may imply non-smooth paths with large discrete jumps from $q_0$ to $p$.
For instance $q_t(x)=(1-e^{-t})p(x)+e^{-t}q_0(x)$ exhibits this type of exponential convergence, 
but corresponds to intermediate distributions that are trivial interpolations
between source and target distributions, and don't correspond to a meaningful smooth path from source to target,
in the spirit of the Benamou-Brenier dynamic transport. See Figure \ref{fig:1d} for an illustration.
 
\textbf{Case 2: $\lambda>0$ Regularized Sobolev Discrepancy Flows.} In this case Corollary \ref{corr:descent} suggests the following non linear dynamic of the MMD :
\begin{align*}
\frac{1}{2}\frac{d}{dt}\text{MMD}^2(\nu_p,\nu_{q_{t}})&=-(\text{MMD}^2(\nu_p,\nu_{q_{t}})-\lambda \pazocal{S}^2_{\mathcal{H},\lambda}(\nu_p,\nu_{q_{t}}))\\
&\leq 0.
\end{align*}
Since $g(t)=\text{MMD}^2(\nu_p,\nu_{q_{t}})$ is decreasing and positive (bounded from below) it has a finite limit $L$ as $t\to \infty$. 
When $g(t)$ reaches this limit at $t=t_0$ we have $\frac{dg(t)}{dt}|_{t=t_0}=0$, and the graph of $g(t)$ remains constant, $g(t)=g(t_0)=L$ for $t\geq t_0$. Hence $\lim_{t\to \infty} \text{MMD}^2(\nu_p,\nu_{q_{t}})=L=g(t_0)$.
 
We make here the following assumption on the target distribution $\nu_p$ that ensures that this limit $L$ is zero.

 \textbf{Assumption (A)}: \emph{For any measure $\nu_q $, such that $\delta_{p,q}=\boldsymbol{\mu}(\nu_p)-\boldsymbol{\mu}(\nu_q)\neq 0$, $\delta_{p,q} \notin Null \left(  D(\nu_q)\right)$ .} 
Assumption  (A) means that we have: $D(\nu_q)(\boldsymbol{\mu}(\nu_p)-\boldsymbol{\mu}(\nu_q))\neq 0$, for all $q$ such that $\delta_{p,q}\neq 0$. This  is a reasonable assumption and it is usually met in practice. 

We show in Proposition \ref{pro:conv} in Appendix \ref{app:convReg} that under assumption A,   the regularized continuous Sobolev descent is convergent in the MMD sense:  $\lim_{t\to \infty} \text{MMD}^2(\nu_p,\nu_{q_{t}})=0$.

Now if Assumption (A) does not hold, Sobolev Descent may stall at a $\nu_{q_{t_0}}$ where $\delta_{p,q_{t_0}} \in Null(D(\nu_{q_{t_0}}))$, and $\text{MMD}^2(\nu_p,\nu_{q_{t}}) \to  \text{MMD}^2(\nu_p,\nu_{q_{t_0}}) \neq 0$ as $t \to \infty$.

\textbf{Infinite dimensional  RKHS, Characteristic kernel and Convergence in distribution of Sobolev Descent.} For $\lambda>0$, Theorem \ref{theo:PertubationMMDSobolev} holds true when $\Phi$ corresponds to an infinite dimensional feature map of a characteristic kernel $k$, without any further assumptions (The case $\lambda =0$ needs more care, and is tackled in Appendix \ref{app:InfDim}) . For $\lambda>0$, under assumption (A) and for a characteristic kernel, the convergence of Sobolev descent in the MMD sense $\text{MMD}(\nu_{p},\nu_{q_{t}})\to 0$ as $t\to \infty$, implies  convergence in distribution : $\nu_{q_{t}} \overset{D}{\to} \nu_{p}$.

\textbf{Damping effect of Regularization.} While the MMD decreases at each time step, 
 the regularization slows down the decrease of MMD by a factor proportional to the regularized Sobolev discrepancy $\lambda \pazocal{S}^2_{\mathcal{H},\lambda}(\nu_p,\nu_{q_{t}})$.
Therefore regularization here is not only playing a computational role that stabilizes computation, it is also playing the role of a \emph{damping}. This \emph{damping} is desirable as it favors smoother paths between $q_0$ and $p$, i.e paths that deviates from the exponential regime in the un-regularized case.
Hence we obtain tunable paths via regularization that favors smoother transitions from source to target (Fig \ref{fig:1d}).

\subsection{ Principal Transport Directions}\label{sec:principal}
\vskip -0.15in
In this section we shed light on how the flow of the Sobolev critic $\nabla_{x}u^\lambda_{p,q}(x)$ transports particles from $q$ to $p$.
To simplify notation we will omit the subscript $p,q$ in this section but keep in mind that $u^\lambda(x)$ is to be determined for any intermediate state $q_t$.
Let $(\lambda_j,\boldsymbol{d_j})$ be eigenvalues and eigenvectors of $D(\nu_q)$ (Eq \ref{eq:D}) with $\lambda_j\geq 0$ descending.
We can now think of $\nabla_x d_j(x)$ as \textbf{principal transport directions},
where
$\nabla_x d_j(x)=[J\Phi(x)]\boldsymbol{d_j}$.
This viewpoint becomes clear when we decompose the direction with which the particles move, i.e. the gradient of the critic $u^\lambda(x)$, over this basis $\nabla_x d_j(x)$.
It is easy to see that:
$u^{\lambda}(x)
= \sum_{j=1}^{m}\frac{1}{\lambda_j+\lambda}\scalT{\boldsymbol{d_j}}{\boldsymbol{\mu}(\nu_p)-\boldsymbol{\mu}(\nu_q)}\scalT{\boldsymbol{d_j}}{\Phi(x)}, $
 and
 
$\nabla_x u^{\lambda}(x)
=\sum_{j=1}^{m} \nabla_x u^\lambda_j(x)  
=\sum_{j=1}^{m}\frac{1}{\lambda_j+\lambda}\scalT{\boldsymbol{d_j}}{\boldsymbol{\mu}(\nu_p)-\boldsymbol{\mu}(\nu_q)}\nabla_x d_j(x).$

The Sobolev critic flow  $\nabla_x u^{\lambda}(x)$ is thus decomposed on those principal transport directions, where each principal transport direction $\nabla_x d_{j}(x)$ is weighted by $\frac{1}{\lambda_j+\lambda} a_j$  where $a_j =\scalT{\boldsymbol{d_j}}{\boldsymbol{\mu}(\nu_p)-\boldsymbol{\mu}(\nu_q)}$. Let us first look at the meaning of $a_j$: if $a_{j}>0$ this mean that this principal transport direction implies the correct motion advecting $q$ to $p$ (positively aligned with the difference of mean embeddings). 
On the other hand the term $\frac{1}{\lambda_j+\lambda}$, explains the role of  regularization. Regularization is introducing a spectral filter on principal transport directions by weighing down directions with low eigenvalues.
Those directions correspond to high frequency motions resulting in discrete jumps and discontinuous paths.
Filtering them out with regularization parameter $\lambda$ ensures smoother transitions in the probability path.
See Figure \ref{fig:Transp}  for an illustration. More in Appendix \ref{app:regu_smoothing}.

\subsection{Sobolev Descent as proxy for GANs} In this section we show how Sobolev descent can be seen as a proxy to GANs~\cite{GANoriginal} that is more  amenable to analysis. 
In Sobolev GAN \cite{SobolevGAN}, the critic  between the current implicit distribution of the  generator $G_{\theta}$ and the target distribution of real data $\mathbb{P}$ is  updated. Then the generator is updated via gradient descent on the parameter space $\theta$. This is similar to Sobolev descent,
with the difference that GAN has a generator that is updated with the gradient flow of the critic, while Sobolev descent transports explicitly particles along that flow.

More formally Sobolev GAN   \cite{SobolevGAN}  has the following updates:
$f_{t}=\argmax\{\mathbb{E}_{x\sim \mathbb{P}}f(x)-\mathbb{E}_{x\sim q_{t}}f(x): f\in \mathcal{H}, \mathbb{E}_{x\sim q_t} ||\nabla_x f(x) ||^2\leq1 \}$ where $q_{t}$ is the distribution of the generator $G_{\theta_{t}}(z),z\sim p_{z}$. Using a continuous form of gradient descent on the generator parameter $\theta$, we can write by the chain rule :  
 $$d\theta_{t}= \mathbb{E}_{\tilde{z}\sim p_z}\left[\frac{\partial G_{\theta}(\tilde{z})}{\partial \theta}\nabla_{x}f_{t}(G_{\theta}(\tilde{z}))\right]_{\theta=\theta_t} dt,$$ 
 where $\frac{\partial G_{\theta}(\tilde{z})}{\partial \theta} \in \mathbb{R}^{|\theta|\times d}$ is the Jacobian matrix.
 
In order to match the particles intuition of Sobolev descent, we show here how to go from generator to particles. Fix $z$ and set $X_{t}=G_{\theta_t}(z)$. $X_{t}$ defines moving particles as $\theta_{t}$ is updated . 
Our goal  is to see if  the velocity of  particles $X_{t}$ produced by the generator in Sobolev GAN has similar behavior to the particles velocities in Sobolev descent. Using  the chain rule we have : $dX_{t}=   \frac{\partial G_{\theta}(z)}{\partial \theta}^{\top}\Big|_{\theta=\theta_t}d\theta_{t}$.  Finally plugging the expression of $d\theta_{t}$ we have $dX_{t}= $
\vskip-0.3in
\begin{equation}
\mathbb{E}_{\tilde{z}\sim p_z}(\frac{\partial G_{\theta}(z)}{\partial \theta}^{\top}\frac{\partial G_{\theta}(\tilde{z})}{\partial \theta}\Big|_{\theta=\theta_{t}}\nabla_{x}f_{t}(G_{\theta_t}(\tilde{z})))dt
\label{eq:sobolevGANParticles}
\end{equation}
\vskip-0.20in
If $G_{\theta}$ satisfies 
 $(\frac{\partial G_{\theta}(z)}{\partial \theta}^{\top}\frac{\partial G_{\theta}(\tilde{z})}{\partial \theta})=\delta(z-\tilde{z}) I_d$ we recover the particles velocities of Sobolev descent:
$dX_{t}= p(z)\nabla_{x}f_t(X_t)dt,$ and our  convergence analysis immediately applies to Sobolev GAN.
Of course one needs to weaken the assumptions on $G_{\theta}$ to Lipschitzity of  the Jacobian $\frac{\partial G_{\theta}(.)}{\partial \theta}$ in the latent space variable $z$ and to carry further the analysis, we leave that for a future work. Our analysis of Sobolev descent suggests to consider gradient descent on the parameter space of the generator in GAN  as a gradient flow on the  probability space, corresponding to particles moving with a non linear Mckean Vlasov process \cite{Vlasov:1968} given in \eqref{eq:sobolevGANParticles}, and may allow under suitable conditions a theoretical understanding of GAN convergence complementing related works such as the ones of \cite{BottouALO17}. 


\section{Relation to Previous Work}\label{sec:previouswork}
Dynamic OT of Benamou-Brenier \cite{dynamicTransport} and Stein descent \cite{steindescent} are the closest to Sobolev Descent.
The Benamou-Brenier formulation and Sobolev Descent minimize two related forms of kinetic energy 
in order to find paths connecting source and target distributions.
Table \ref{tab:comparison} in Appendix \ref{app:diagram_ot} summarizes those main differences.
In the Stein method \cite{GorhamM15,LiuLJ16,SteinGretton, GorhamDVM16,GorhamM17}, one of the measures $\nu_p$ is assumed to have a known density function $p$ and we would like to measure the fidelity of samples from $\nu_q$ to the likelihood of $p$.
The Stein discrepancy is obtained by applying a differential operator $T(p)$ to a vector valued function $\varphi: \pazocal{X}\to \mathbb{R}^d$, where $T(p)\varphi(x)=\scalT{\nabla_x\log(p(x))}{\varphi(x)}+\text{div}(\varphi(x)).$
The Kernelized Stein Discrepancy is defined as follows:
$\mathbb{S}(\nu_p,\nu_q)=\sup_{\varphi}\{\mathbb{E}_{x\sim \nu_{q}}T(p)\varphi(x):\varphi_j \in \mathcal{H},\sum_{j=1}^d \nor{\varphi_j}^2_{\mathcal{H}}\leq 1\}$.
Let $\varphi^*_{p,q}$ be the optimal solution.
Variational Stein Descent of \cite{steindescent} uses  $\varphi^*_{p,q}$ as a velocity field to transport particles distributed according to $\nu_q$
to approximate the target $\nu_p$.
This  constructs paths reducing the KL divergence \cite{Stein}.

\section{Algorithms and Experiments}\label{sec:alg}
\textbf{Algorithms.}
We specify here the regularized Sobolev Descent for empirical measures $\hat{\nu}_p$ and $\hat{\nu}_q$,given finite samples from $p$ and $q$: $\{x_i,i=1\dots N, x_i \sim p\}$, and $\{y_i,i=1\dots M, y_i \sim q\}$.

\textbf{Empirical Regularized Kernel Sobolev Descent with Random Fourier Features.} We consider the finite dimensional RKHS induced by random Fourier features \cite{rf}($\Phi(x)=\cos(Wx+b),W_{ij}\sim \mathcal{N}(0,1/\sigma^2), b_i \sim \text{Unif}[0,2\pi] $) . 
The empirical descent consists in using the estimate $\hat{u}^{\lambda}_{p,q}$ in Equation \eqref{eq:path}.
For $\varepsilon>0$, we have the following iteration, for $\ell \geq 1$ and all current positions of source particles $ i=1,\dots M$:
$$x^{\ell}_i = x^{\ell-1}_i +\varepsilon \nabla_x \hat{u}^{\lambda}_{p,q_{\ell-1}}(x^{\ell-1}_i)$$ 
with $\hat{\nu}_{q_{\ell-1}}(dx)=\frac{1}{M}\sum_{i=1}^M \delta(x-x^{\ell-1}_i) dx$, the empirical measure of particles $\{x^{\ell-1}_i,i=1\dots M\}$, initialized with source particles $\{x^0_i=y_i,i=1\dots M\}$, and $\hat{u}^{\lambda}_{p,q_{\ell-1}}$ is the optimal critic of the empirical RKSD between empirical measure $\hat{\nu}_p$ and $\hat{\nu}_{q_{\ell-1}}$. 
The empirical regularized Kernel Sobolev Descent can be written as follows: for $l=\{1\dots L\}$:
$$\boldsymbol{\hat{u}}^{\lambda}_{p,q_{\ell-1}}= (\hat{D}(\hat{\nu}_{q_{\ell-1}})+\lambda I_m)^{-1} \left(\hat{\boldsymbol{\mu}}(\hat{\nu}_p)-\hat{\boldsymbol{\mu}}(\hat{\nu}_{q_{\ell-1}}) \right)$$
$$x^{\ell}_i = x^{\ell-1}_i +\varepsilon [J\Phi(x^{\ell-1}_i )]\boldsymbol{\hat{u}}^{\lambda}_{p,q_{\ell-1}} ,\forall i=1,\dots M.$$
The Empirical Sobolev Descent is summarized in Algorithm \ref{alg:KSD}, and the smoothness of the paths is controlled via the regularization parameter $\lambda$.

\textbf{Neural Sobolev Descent.} 
Inspired by the success of Sobolev GAN \cite{SobolevGAN} that uses neural network approximations to estimate the Sobolev critic, we propose Neural Sobolev Descent.
In Neural Sobolev Descent the critic function between $\nu_{q_t}$ and $\nu_{p}$ is estimated using a neural network $f_{\xi}(x)=\scalT{v}{\Phi_{\omega}(x)}$, where $\xi=(v,\omega)$ are the parameters of the neural network that we learn by gradient descent.
We follow \cite{SobolevGAN} in optimizing the parameters of the critic via an augmented Lagrangian.
The particles descent is the same as in the Kernelized Sobolev Descent.
Gradient descent on the parameters of the critic between updates of the particles resumes from the previous parameters (warm restart).
Neural Sobolev Descent is summarized in Algorithm \ref{alg:NSD}.
Note that when compared to Sobolev GAN this descent replaces the generator with particles.
It is worth mentioning that regularization in the neural context is obtained via early stopping, i.e the number of updates $n_c$ of the critic.
Early stopping is known as a regularizer for gradient descent \cite{yao2007earlystopping}.
We will see that the smoothness of the paths is controlled via $n_c$.
Note that GAN stabilization  through early stopping (small critic updates) has been empirically observed \cite{arjovsky2017towards,fedus2017many}. Our analysis suggests that this induces smoother paths for GANs.

\textbf{Experiments.} We confirm our theoretical findings on regularized Sobolev descent on a synthetic example highlighting  the crucial role of regularization in smooth paths convergence.
We then baseline Sobolev descent versus classical OT algorithms on the image color transfer problem.
We show well-behaved trajectories of Sobolev descent in shape morphing thanks to smooth regularized paths.
\vskip -0.2in
\begin{figure}[ht!]
\centering
\includegraphics[width=0.5\textwidth]{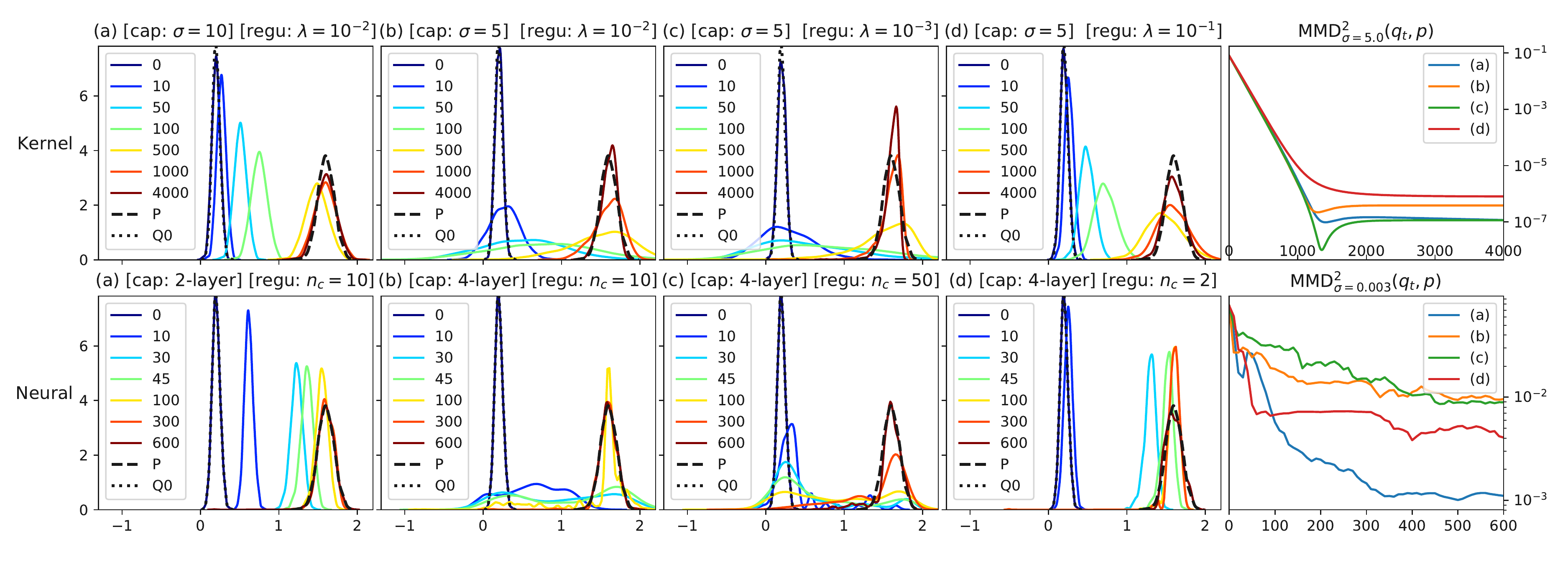}
\vskip -0.2in
\caption{Moving 1000 samples of a 1D gaussian $\nu_{q_0}=\mathcal{N}(0.2, \sigma=0.005)$ to $\nu_p=\mathcal{N}(1.6, \sigma=0.1)$ with Kernel (top row) and Neural (bottom row) Sobolev Descent.
Columns have similar properties between kernel and neural variants in terms of capacity of the model and regularization of the descent:
(a) low capacity, (b) high capacity,
(c) high capacity with decreased regularization and (d) high capacity with increased regularization.
}
\label{fig:1d}
\vskip -0.1in
\end{figure}

\textbf{Synthetic 1D Gaussians.}
Figure \ref{fig:1d} shows Sobolev descent trajectories on a toy 1D problem, where both source and target are 1D Gaussians.
Note that the Benamou-Brenier solution would be a \emph{smooth} trajectory of normal distributions, where both the mean and standard deviation linearly interpolate between $q_0$ and $p$.
Given $1000$ samples from $q_0$ and $p$, we show in Figure \ref{fig:1d} results of both kernel and neural Sobolev descent, where we plot kernel density estimators of densities at various time steps in the descent.
We show the results of the descent for varying capacity of the function space ($\sigma$ for the random features kernel, number of layers for Neural), and various regularization parameters
($\lambda$ Tikhonov regularization for Kernel and $n_c$ early stopping for Neural).
Column (a) shows a regularized low capacity model achieving good approximation of the Benamou-Brenier optimal trajectory, where the data remains concentrated and smoothly moves to the target distribution.
Column (b) shows a higher capacity model which blurs out the distribution before converging to $p$,
where column (c) we even further decrease the regularization (smaller $\lambda$, bigger $n_c$) confirming the undesirable interpolation behavior which is predicted by the theory in the un-regularized case.
Note that even in the Neural SD case this happens, corresponding to high frequency critic gradient behavior.
In column (d) we increase the regularization on the high capacity model, achieving again a behavior that is closer to the optimal, without blurring or interpolation.
This confirms the damping effect of regularization, filtering out the high frequency gradients.
This can be also seen in the MMD plot in the last column.
Figure \ref{fig:convergence} in Appendix \ref{app:conv} gives similar results on morphing.
\vskip -0.1 in
\begin{figure}[ht!]
\centering
\includegraphics[width=0.5\textwidth]{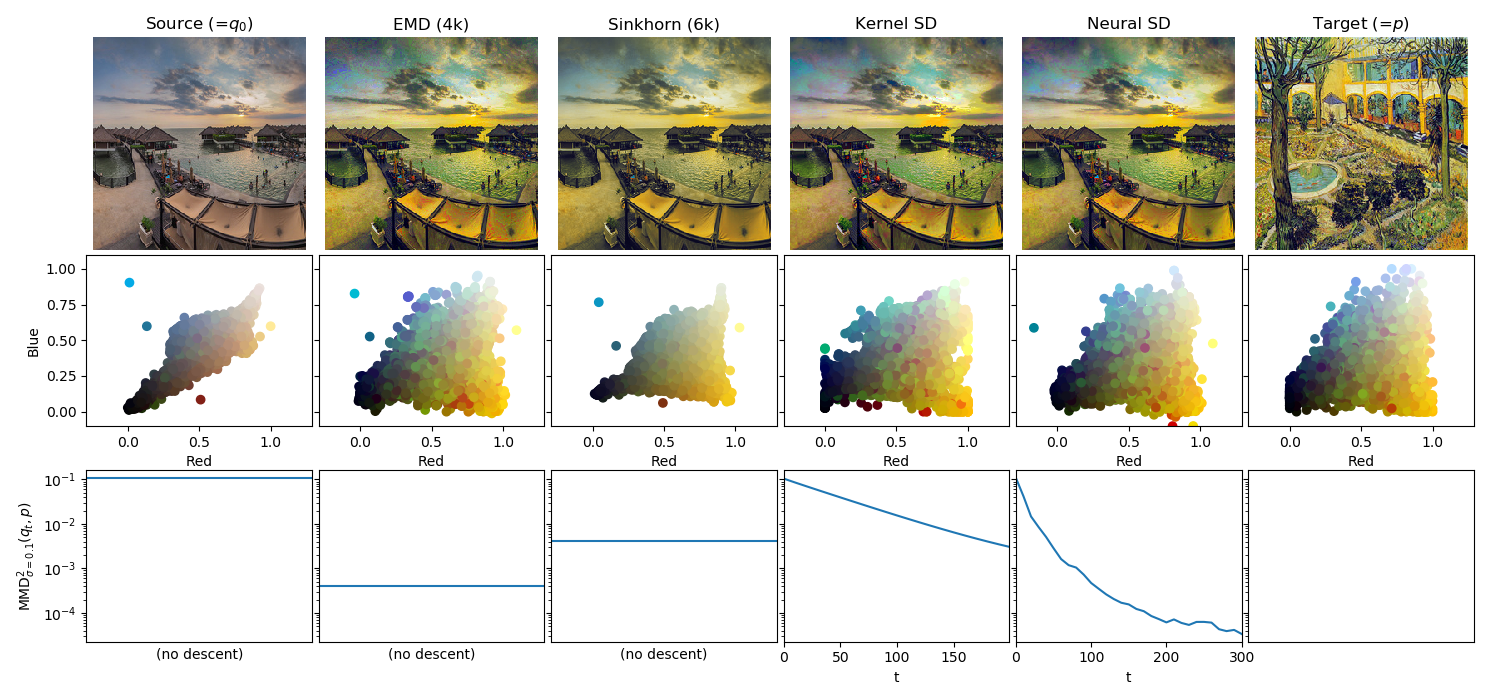}
\vskip -0.1in 
\caption{
  Color transfer.
  We compare Earth Mover Distance solved with linear programming on 4k samples, Sinkhorn on 6k samples with regularization $\varepsilon=1e^{-2}$,
  Kernel SD with $\lambda=1e^{-2}$ and $\sigma=0.1$ at $t=200$, and Neural SD at $t=300$.
  The bottom row shows progress during the descent by computing the MMD$(\nu_{q_t},\nu_p)$ with bandwidth $\sigma=0.1$ using 300 random Fourier features.
  Neural descent has a clear computational advantage over OT alternatives, which alleviates the need for subsampling and out of sample interpolation
  (which explains the high MMD values even for EMD).
}
\label{fig:coloring}
\vskip -0.1in 
\end{figure}

\begin{figure*}[ht!]
\hfill
\begin{minipage}{.35\textwidth}
\centering
\includegraphics[width=\textwidth]{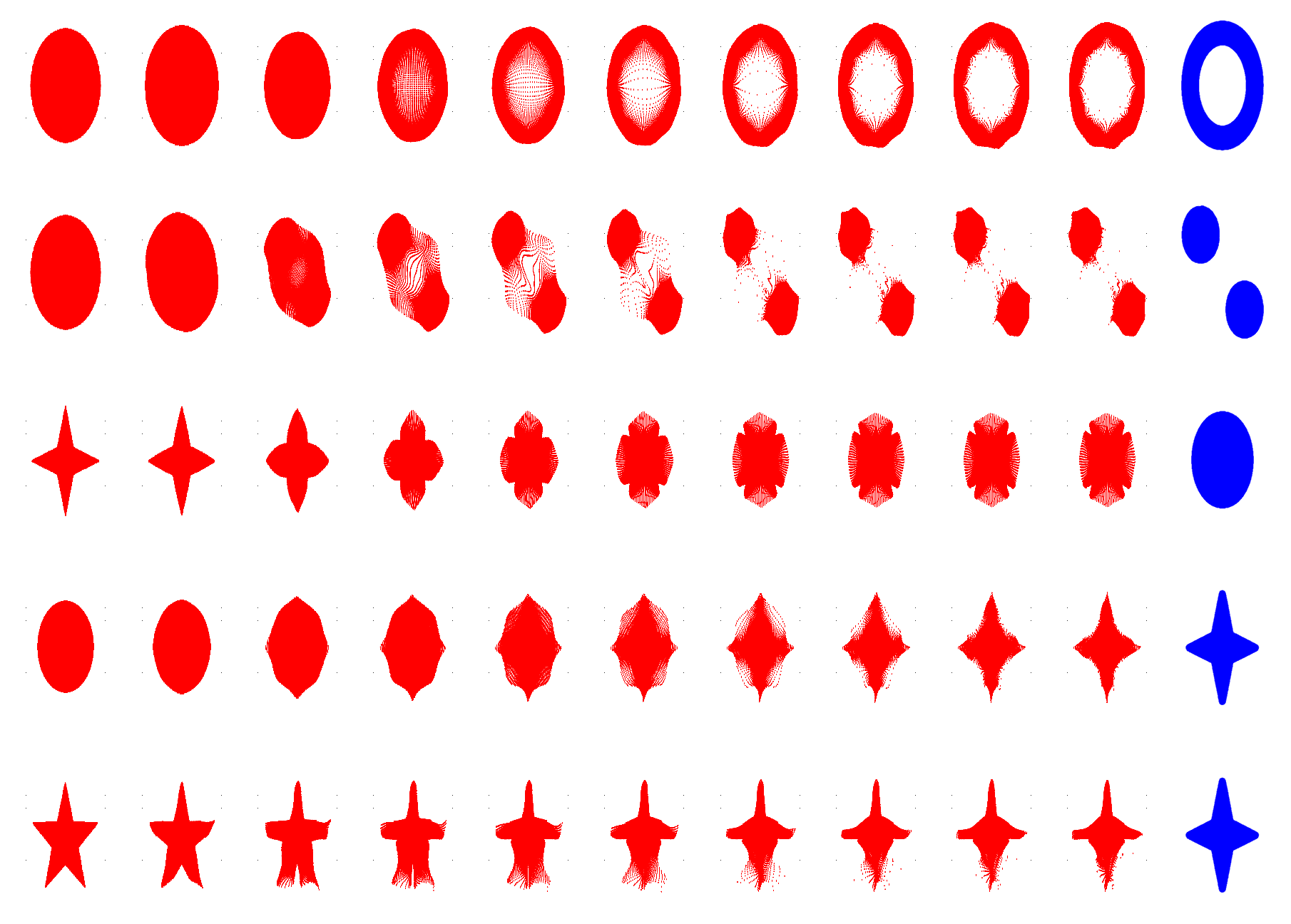}
\captionof{figure}{Morphing between several shapes using Kernelized Sobolev Descent.
Intermediate steps are intermediate particles states of the Sobolev descent.
Last column in the output of Kernelized Sobolev Descent.}
\label{fig:morphing_ksd}
\end{minipage} \quad\quad\quad \hfill
\begin{minipage}{.45\textwidth}
\centering
\includegraphics[height=0.7\textwidth]{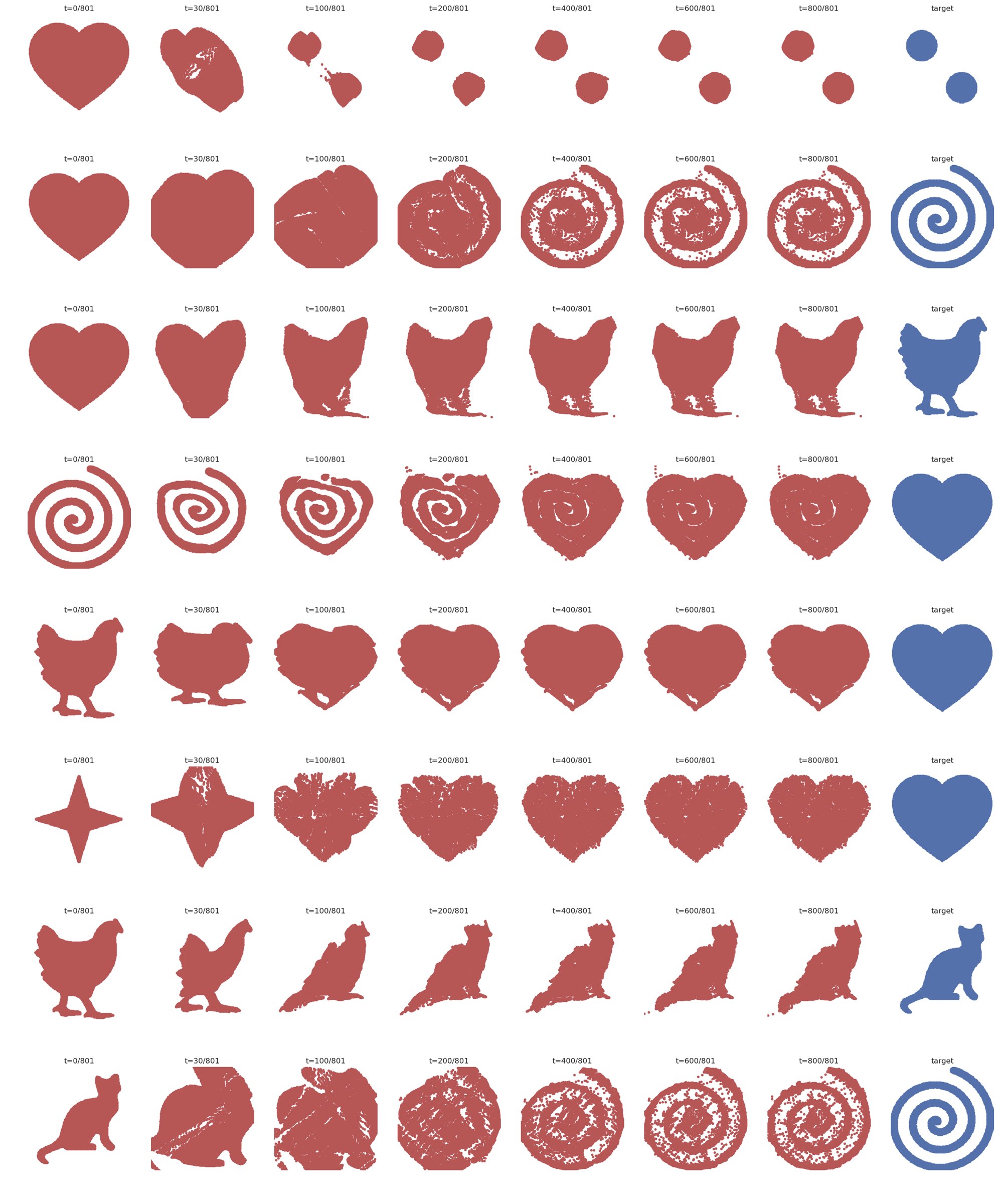}
\captionof{figure}{Morphing between several shapes using Neural Sobolev Descent.
The descent is performed using a critic modeled by a simple 3-layer MLP.
}
\label{fig:morphing_nsd}
\end{minipage}
\vskip -0.25in
\end{figure*}
\begin{figure}[ht!]
\centering
\begin{minipage}{.4\textwidth}
\centering
\includegraphics[width=0.5\textwidth]{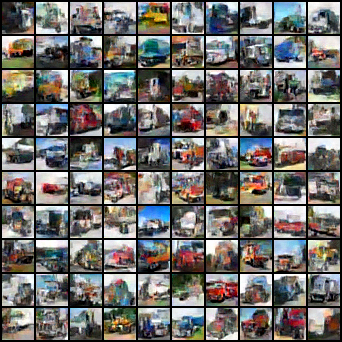}
\captionof{figure}{Particles (Images) of Neural Sobolev Descent at convergence, when the target distribution is the trucks class of CIFAR 10 and the Sobolev critic is a learned CNN. }
\label{fig:sampleDescent}
\end{minipage}
\begin{minipage}{.5\textwidth}
\centering
\includegraphics[width=0.5\textwidth]{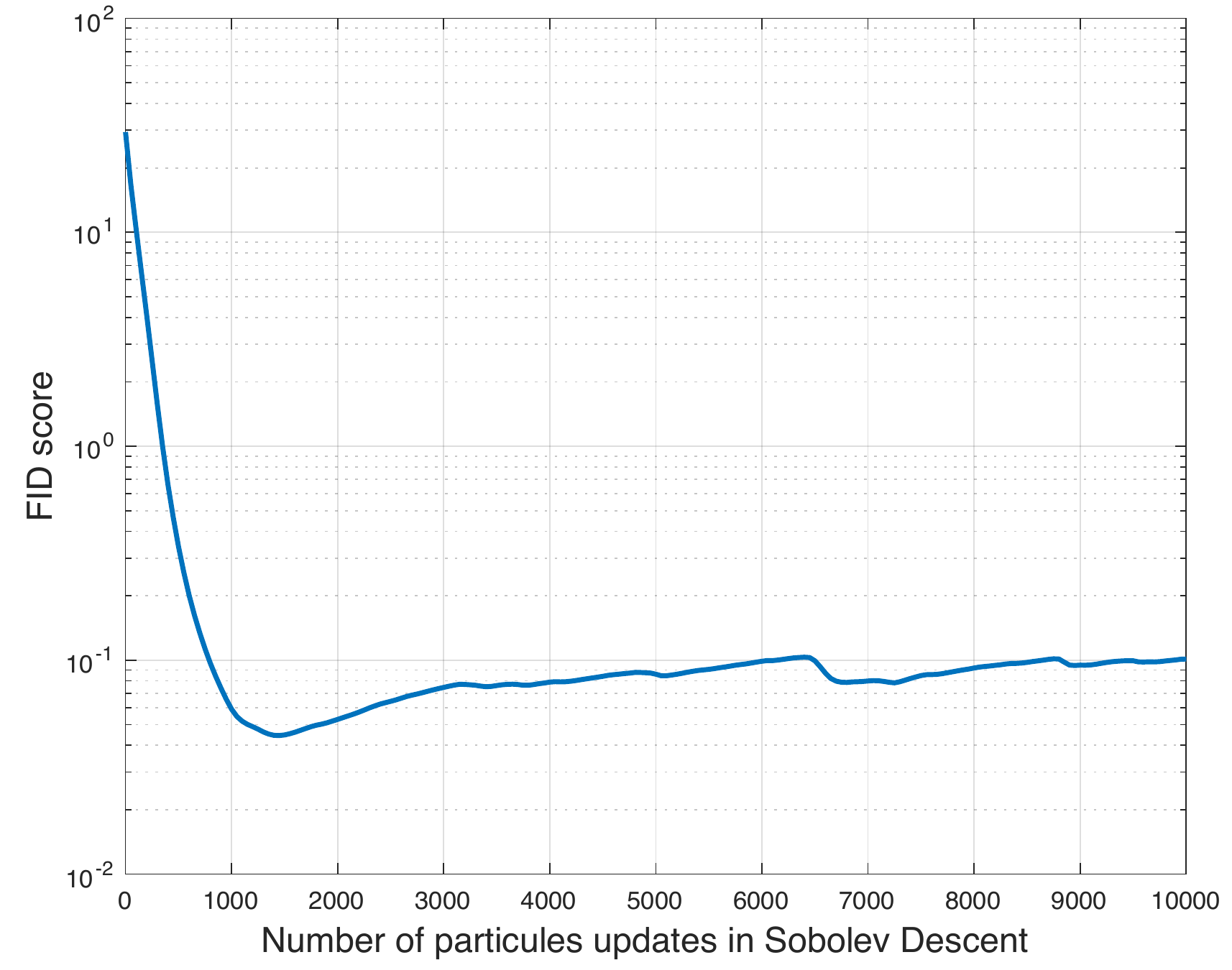}
\captionof{figure}{Frechet Inception Score (FID) of the particles produced by neural Sobolev descent as the descent progresses. FID is computed using the features from the second max pooling layer of the Inception v3 net (192-dim), by comparison against the truck class.}
\label{fig:FID}
\end{minipage}
\vskip -0.1in 
\label{fig:trucks}
\vskip -0.2in 
\end{figure}

\textbf{Image Color Transfer.} 
We consider the task of image color manipulation where we would like for an image $A$ to match the color distribution of an image $B$.
More formally, consider colored images Source and Target which we see as defining 3-dimensional probability distributions $\nu_{q_0}$ and $\nu_p$, where every pixel is a sample:
$\{x_1, x_2, \cdots x_N\} \sim \nu_{q_0}$ and $\{y_1, y_2, \cdots y_N\} \sim \nu_p$ and $N=256\times256=66k$ the resolution.
We move the samples using Kernel Sobolev Descent and Neural Sobolev Descent and analyze the distributions $q_t$.
We provide in Figure \ref{fig:coloring} the results of our proposed algorithm on the task of image color transfer,
comparing against results obtained with static Optimal Transport\footnote{
We follow the recipe of \cite{ferradans2013regularized} as implemented in the POT library \cite{flamary2017pot} where we subsample for computational feasibility, then use interpolation for out-of-sample points.}.
We show scatter plots after subsampling 5k points at random and display them on the (R,B) channels.
In Appendix \ref{app:exp_additional} Figures \ref{fig:coloring_1_sigma} and \ref{fig:coloring_1_timetrace} we show
the final MMD in function of rbf bandwidth $\sigma$ and the evolution of the $q_t$ distribution during the descent.

\textbf{Shape Morphing with Sobolev Descent.}

We use Sobolev descent for morphing between shapes.
The source distribution is the distribution of points $x\in\R^2$ sampled uniformly from a shape $A$, that we need to move to become shape $B$.
Such type of morphing has been considered in the Wasserstein Barycenter framework \cite{Solomon2015,CuturiD14}.
Figure \ref{fig:morphing_ksd} shows the result of Kernelized Sobolev Descent (Algorithm \ref{alg:KSD}) transforming between a source shape $\nu_q$ and a target shape $\nu_p$,
using random fourier features for $m=100$ and $L=600$, $\varepsilon=0.01$ and $\lambda=0.01$.
We see that Kernelized Sobolev Descent morphs the shapes as the number of iterations approaches $L=600$. Figure \ref{fig:morphing_nsd} shows Neural Sobolev Descent morphing between source shapes and target shapes.
The first column is the source shape and last column is the target shape, in between columns are intermediate outputs of the Neural Sobolev Descent.
Neural Sobolev Descent converges even on complex and unrelated shapes.
Appendix \ref{app:exp_morphing} provides the implementation and training details, and
visualizes the critic $f_\xi(x)$ during the descent (Figure \ref{fig:morphing_nsd_quivers}).
Code is available on \url{https://bit.ly/2GtWXsY}. Videos of shapes morphing are available on \url{https://goo.gl/X4o8v6}.

\textbf{High Dimensional Experiments: Transporting Noise to Images.}
We use neural Sobolev descent to transport uniform noise to the 5000 images in CIFAR10 labeled truck, similar to a typical GAN setup.
The Sobolev critic architecture is a DCGAN discriminator architecture \cite{dcgan}. We see in Fig \ref{fig:sampleDescent} that Sobolev descent converges and produces samples similar to the images from a trained GAN.
The FID score \cite{heusel2017gans} along the descent is given in Fig \ref{fig:FID}. This experiment confirms qualitatively and quantitatively our theoretical findings on Sobolev descent as a simplified proxy for GANs.
\section{Conclusion}
\vskip -0.19in
We introduced Sobolev descent on particles as a simplified proxy to GAN training.
Sobolev descent constructs paths of distributions which minimize a kinetic energy, similar to dynamical Optimal Transport.
We highlighted its convergence, its capacity in modeling high dimensional distributions and the crucial role of regularization in obtaining smooth transition paths by filtering out high frequency gradients.
Our work sheds light on gradient based learning of GANs such as Sobolev GAN \cite{SobolevGAN},
that can be seen as a dynamic transport rather than the static as popularized by WGAN \cite{WGAN}.
Our analysis explains GAN stabilization  through early stopping (small updates of critic) \cite{arjovsky2017towards,fedus2017many}  as a regularization on the critic,  inducing smoother paths to equilibrium.


\bibliography{refs,simplex}
 \bibliographystyle{unsrt}

\appendix
\onecolumn
\begin{center}
\large{\textbf{Sobolev Descent: Supplementary Material}} 
\end{center}

\section{Background on Sobolev Discrepancy}
Recently \cite{mroueh2018kernelsobdiscrepTechRep} showed that  Sobolev Discrepancy is rooted in the optimal transport literature and is known as the \emph{homegeonous weighted negative Sobolev norm} $\nor{.}_{\dot{H}^{-1}(\nu_q)}$, \cite{Villani,peyre2016comparison,peyre2017computational}.
Indeed the weighted negative Sobolev norm is defined as follows, for a any signed measure $\chi$:
 \vskip -0.2in 
$$\nor{\chi}_{\dot{H}^{-1}(\nu_q)}=\sup_{f}\Big\{\left|\int_{\pazocal{X}} f(x)d\chi(x)\right| \Big\}$$
$$ \text{s.t } f \in W^{1,2}_0(\pazocal{X},\nu_{q}), 
 \underset{x\sim \nu_{q}}{\mathbb{E}}\nor{\nabla_xf(x) }^2 \leq 1
.$$
As pointed out in \cite{mroueh2018kernelsobdiscrepTechRep} it is easy to see that:
$\pazocal{S}(\nu_{p},\nu_{q})= \nor{\nu_{p}-\nu_{q}}_{\dot{H}^{-1}(\nu_q)}$.
The norm $\nor{.}_{\dot{H}^{-1}(\nu_q)}$ plays a fundamental role in dynamic optimal transport \cite{dynamicTransport} since it linearizes the Wasserstein $W_2$ distance:
$W_2(\nu_{q},\nu_{q}+\varepsilon \chi)= \varepsilon \nor{\chi}_{\dot{H}^{-1}(\nu_q)}+o(\varepsilon).$
For more details on the Sobolev Discrepancy and its connection to optimal transport we refer the reader to \cite{mroueh2018kernelsobdiscrepTechRep} and references there in.

\section{Theory of Sobolev Descent in Finite dimensional RKHS }\label{app:Proofs}
\textbf{MMD as a Functional over probabilities and its First variation.} In order to characterize the variation in the MMD distance under small perturbations of the source distribution, we think of the MMD as a functional over the probability space $\mathcal{P}(\pazocal{X})$. 
The following definition of the first variation of functionals over Probability is a fundamental tool in our analysis.
The reader is referred to \cite[Chapter 7]{santambrogio2015optimal} for more context on first variations and gradient flows in optimal transport.

\begin{definition}[First variation of Functionals over Probability]
We shall fix in the following a measure $\nu_p$ and perturb $\nu_q$ with a perturbation $\chi$ so that $\nu_q+\varepsilon \chi$ belongs to $\mathcal{P}(\pazocal{X})$ for small $\varepsilon$ (We have necessarly $\int d\chi=0$).
Let $F$ be a functional: $\mathcal{P}(\pazocal{X})\times \mathcal{P}(\pazocal{X})\to \mathbb{R}^+$.
We treat $F(\nu_p,\nu_q)$, as a functional over probability in its second argument and compute its first variation as follows:
$$\frac{d}{d\varepsilon}F(\nu_p,\nu_q+\varepsilon \chi)\Big|_{\varepsilon =0}=\lim_{\varepsilon \to 0} \frac{F(\nu_p,\nu_q+\varepsilon \chi)-F(\nu_p,\nu_q) }{\varepsilon}:= \int \frac{\delta F }{\delta \nu_q}\left(\nu_p,\nu_q\right)d\chi$$
\end{definition}

\begin{proposition}[Perturbation of the MMD] Let $\varphi^*_{p,q}= \frac{\boldsymbol{\mu}(\nu_p)-\boldsymbol{\mu}(\nu_q)}{\nor{\boldsymbol{\mu}(\nu_p)-\boldsymbol{\mu}(\nu_q)}_{\mathcal{H}}}$ the witness function of the MMD distance:
$$\text{MMD}(\nu_p,\nu_q)=\sup_{\varphi \in \mathcal{H},\nor{\varphi}_{\mathcal{H}}\leq 1} \int_{\pazocal{X}}\varphi d(\nu_p-\nu_q).$$
We have the following first variation result:
$$\frac{d}{d\varepsilon}\text{MMD}(\nu_p,\nu_q+\varepsilon\chi)\Big|_{\varepsilon=0}=-\int \varphi^*_{p,q}d\chi. $$
\end{proposition}
\begin{proof} This result is direction application of Proposition 7.16 in Chapter 7 of Optimal Transport for applied Mathematicians book.
\end{proof}

\begin{lemma}[Perturbation of \text{MMD}$^2$] We have the following first variation for the MMD distance:
$$\frac{d}{d\varepsilon}\text{MMD}^2(\nu_p,\nu_q+\varepsilon\chi)\Big|_{\varepsilon=0}=-2\int \delta_{p,q}(x)d\chi(x), $$
where $\delta_{p,q}(x)=\scalT{\boldsymbol{\mu}(\nu_p)-\boldsymbol{\mu}(\nu_q)}{\Phi(x)}.$
\end{lemma}
\begin{proof}
$$\text{MMD}^2(\nu_p,\nu_q)=\nor{\boldsymbol{\mu}(\nu_p)-\boldsymbol{\mu}(\nu_q)}^2$$
We extend the kernel mean embedding definition here to signed measures $\chi$ and note:
$$\boldsymbol{\mu}(\chi)=\int_{\pazocal{X}}\Phi(x) d\chi(x).$$
\begin{align*}
\text{MMD}^2(\nu_p,\nu_q+\varepsilon \chi )&=\nor{\boldsymbol{\mu}(\nu_p)-\boldsymbol{\mu}(\nu_q)-\varepsilon \boldsymbol{\mu}(\chi)}^2\\
&=\nor{\boldsymbol{\mu}(\nu_p)-\boldsymbol{\mu}(\nu_q)}^2-2\varepsilon \scalT{\boldsymbol{\mu}(\nu_p)-\boldsymbol{\mu}(\nu_q)}{\boldsymbol{\mu}(\chi)}+\varepsilon^2 \nor{\boldsymbol{\mu}(\chi)}^2\\
&=\text{MMD}^2(\nu_p,\nu_q)-2\varepsilon \scalT{\boldsymbol{\delta}_{p,q}}{\boldsymbol{\mu}(\chi)}+\varepsilon^2 \nor{\boldsymbol{\mu}(\chi)}^2\\
&=\text{MMD}^2(\nu_p,\nu_q)-2\varepsilon \int_{\pazocal{X}} \delta_{p,q}(x)d\chi(x) +\varepsilon^2 \nor{\boldsymbol{\mu}(\chi)}^2,\\
\end{align*}
where we noted, $\boldsymbol{\delta}_{p,q}=\boldsymbol{\mu}(\nu_p)-\boldsymbol{\mu}(\nu_q) $ and $\delta_{p,q}(x)=\scalT{\boldsymbol{\delta}_{p,q}}{\Phi(x)}.$
It follows that:
$$\frac{\text{MMD}^2(\nu_p,\nu_q+\varepsilon \chi )-\text{MMD}^2(\nu_p,\nu_q)}{\varepsilon}=-2\int_{\pazocal{X}} \delta_{p,q}(x)d\chi(x) + \varepsilon \nor{\boldsymbol{\mu}(\chi)}^2_{\mathcal{H}}, $$
Taking the limit $\varepsilon \to 0$, we obtain:
$$\frac{d}{d\varepsilon}\text{MMD}^2(\nu_p,\nu_q+\varepsilon\chi)\Big|_{\varepsilon=0}=-2\int_{\pazocal{X}} \delta_{p,q}(x)d\chi(x)$$
\end{proof}

 Let $\psi \in \mathcal{H}$, following \cite{steindescent} we consider infinitesimal transport maps: $T^\varepsilon(x)=x +\varepsilon \nabla_x \psi(x) ,  x \sim \nu_q$.
Let $q$ be the density of $X$ we are interested in the density $q_{T^\varepsilon}$ of $T^\varepsilon(X)$ as $\varepsilon \to 0$.
Consider $\varepsilon$ small so that $\nabla T^\varepsilon(x)= I+ \varepsilon H\psi(x)$ is positive definite, where $H$ is the hessian matrix of $\psi$ (i.e $\varepsilon < \sup_{x\in \pazocal{X}}\frac{1}{ | \lambda_{max}(H\psi(x))|}$).
Therefore we have:
$(T^\varepsilon)^{-1}(x)=x-\varepsilon \nabla_x \psi(x)+o(\varepsilon).$
A first order expansion gives us :
\begin{eqnarray*}
 q_{[T^\varepsilon]}(x)&=&q((T^\varepsilon)^{-1}(x))\text{det}(\nabla_x (T^\varepsilon)^{-1}(x))\\
&=& (q(x)-\varepsilon\scalT{\nabla_x q(x)}{\nabla_x \psi(x)} )\text{det}(I-\varepsilon\nabla^2_x \psi(x))+o(\varepsilon)\\
&=&(q(x)-\varepsilon\scalT{\nabla_x q(x)}{\nabla_x \psi(x)} )(1- trace(\varepsilon\nabla^2_x \psi(x)))+o(\varepsilon)\\
&=& q(x)- \varepsilon(\scalT{\nabla_x q(x)}{\nabla_x\psi(x)}+ q(x)\Delta \psi(x)) +o(\varepsilon)\\
&=& q(x)- \varepsilon( div(q(x)\nabla_x \psi(x)))+o(\varepsilon)
\end{eqnarray*}
Hence we are interested in perturbation of the form $d\chi(x)=-div(q(x)\nabla_x \psi(x))dx$, since it is the first order variation of the density as 
we transport points distributed as $\nu_q$ using the infinitesimal transport map $T^\varepsilon$, for small $\varepsilon$.
Note that $\int_{\pazocal{X}}d\chi(x)=0$.

%

\begin{theorem}[(Thm \ref{theo:PertubationMMDSobolev} restated)]Let $\lambda>0$.
  Let $u^{\lambda}_{p,q}$ the unnormalized solution of the regularized Kernel Sobolev discrepancy between $\nu_p$ and $\nu_q$ i.e $\boldsymbol{u}^{\lambda}_{p,q}=(D(\nu_q)+\lambda I)^{-1}(\boldsymbol{\mu}(\nu_p)-\boldsymbol{\mu}(\nu_q)).$
  Consider $d\chi_{u}(x)=-div(q(x)\nabla_x u^{\lambda}_{p,q}(x))dx$, i.e corresponding to the infinitesimal transport of $\nu_q$ via $T^{\varepsilon}(x)=x+\varepsilon \nabla_x u^{\lambda}_{p,q}(x)$. 
We have the following first variation of the $\text{MMD}^2$ under this particular perturbation:
$$\frac{d}{d\varepsilon}\text{MMD}^2(\nu_p,\nu_q+\varepsilon\chi_{u})\Big|_{\varepsilon=0}=- 2\left(\text{MMD}^2(\nu_p,\nu_q)-\lambda \pazocal{S}^2_{\mathcal{H},\lambda}(\nu_p,\nu_q)\right)\leq 0. $$
\end{theorem}

\begin{proof} [Proof of Corollary \ref{corr:descent}]By Theorem \ref{theo:PertubationMMDSobolev}, noting that for small $\varepsilon$ we have:
$q_{\ell+1}(x)=q_{\ell}(x)-\varepsilon div(q_{\ell}(x)\nabla_x u^{\lambda}_{p,q_{\ell}}(x)).$
\end{proof}

\begin{proof}[Proof of Theorem \ref{theo:PertubationMMDSobolev}]
\begin{align*}
\frac{1}{2}\frac{d}{d\varepsilon}\text{MMD}^2(\nu_p,\nu_q+\varepsilon\chi_u)\Big|_{\varepsilon=0}&=-\int \delta_{p,q}d\chi_u=-\int_{\pazocal{X}}\delta_{p,q}(x)(-div(q(x)\nabla_x u^{\lambda}_{p,q}(x)))dx \\
&=\int_{\pazocal{X}}\delta_{p,q}(x)div(q(x)\nabla_x u^{\lambda}_{p,q}(x))dx\\
&=-\int_{\pazocal{X}} \scalT{\nabla_x \delta_{p,q}(x)}{\nabla_x u^{\lambda}_{p,q}(x)}q(x)dx \text{ (Divergence theorem and zero boundary)}\\
&= -\int_{\pazocal{X}} \boldsymbol{\delta}_{p,q}^{\top}[J\Phi(x)]^{\top}J\Phi(x) \boldsymbol{u}^{\lambda}_{p,q} q(x)dx\\
&= -\scalT{\boldsymbol{\delta}_{p,q}}{\left( \int_{\pazocal{X}} [J\Phi(x)]^{\top}J\Phi(x) q(x) dx\right) \boldsymbol{u}^{\lambda}_{p,q}}\\
&=-\scalT{\boldsymbol{\delta}_{p,q}}{\mathbb{E}_{x\sim \nu_{q}}( [J\Phi(x)]^{\top}J\Phi(x) ) \boldsymbol{u}^{\lambda}_{p,q} }\\
&=- \scalT{\boldsymbol{\delta}_{p,q}}{D(\nu_q)\boldsymbol{u}^{\lambda}_{p,q}} \text{(by definition)}\\
&=- \scalT{\boldsymbol{\delta}_{p,q}}{\left(D(\nu_q)+ \lambda I_m -\lambda I_m\right)\boldsymbol{u}^{\lambda}_{p,q}}\\
& =- \scalT{\boldsymbol{\delta}_{p,q}}{\left(D(\nu_q)+ \lambda I_m\right)\boldsymbol{u}^{\lambda}_{p,q}}+\lambda \scalT{\boldsymbol{\delta}_{p,q}}{\boldsymbol{u}^{\lambda}_{p,q}}\\
\end{align*}
Recall that : $$\left(D(\nu_q)+ \lambda I_m\right) \boldsymbol{u}^{\lambda}_{p,q}=\boldsymbol{\delta}_{p,q},$$
and by definition the regularized Kernel Sobolev Discrepancy we have:
$$\scalT{\boldsymbol{\delta}_{p,q}}{\boldsymbol{u}^{\lambda}_{p,q}}=\pazocal{S}^2_{\mathcal{H},\lambda}(\nu_p,\nu_q),$$
Hence replacing the expressions above we obtain:
\begin{align*}
\frac{1}{2}\frac{d}{d\varepsilon}\text{MMD}^2(\nu_p,\nu_q+\varepsilon\chi)\Big|_{\varepsilon=0}&= - \scalT{\boldsymbol{\delta}_{p,q}}{\boldsymbol{\delta}_{p,q}}+ \lambda \pazocal{S}^2_{\mathcal{H},\lambda}(\nu_p,\nu_q)\\
&=-\text{MMD}^2(\nu_p,\nu_q)+ \lambda \pazocal{S}^2_{\mathcal{H},\lambda}(\nu_p,\nu_q)\\
&= -\left(\text{MMD}^2(\nu_p,\nu_q)-\lambda \pazocal{S}^2_{\mathcal{H},\lambda}(\nu_p,\nu_q)\right)
\end{align*}

Note that:
\begin{align*}
\pazocal{S}^2_{\mathcal{H},\lambda}(\nu_p,\nu_q) = \scalT{\boldsymbol{\delta}_{p,q}}{(D(\nu_q)+\lambda I)^{-1}\boldsymbol{\delta}_{p,q}}&\leq \nor{(D(\nu_q)+\lambda I)^{-1}}_{op}\nor{\boldsymbol{\delta}_{p,q}}^2\\
&\leq \frac{1}{\lambda}\text{MMD}^2(\nu_p,\nu_q),
\end{align*}
It follows that : $$\text{MMD}^2(\nu_p,\nu_q)-\lambda \pazocal{S}^2_{\mathcal{H},\lambda}(\nu_p,\nu_q) \geq 0$$
and 
$$\frac{1}{2}\frac{d}{d\varepsilon}\text{MMD}^2(\nu_p,\nu_q+\varepsilon\chi)\Big|_{\varepsilon=0}= -\left(\text{MMD}^2(\nu_p,\nu_q)-\lambda \pazocal{S}^2_{\mathcal{H},\lambda}(\nu_p,\nu_q)\right)\leq 0.$$
\end{proof}

\subsection{Proofs for Convergence of Continuous Sobolev Descent} \label{app:convReg}
 \textbf{Assumption (A)}: \emph{For any measure $\nu_q $, such that $\delta_{p,q}=\boldsymbol{\mu}(\nu_p)-\boldsymbol{\mu}(\nu_q)\neq 0$, $\delta_{p,q} \notin Null \left(  D(\nu_q)\right)$ .} 

\begin{proposition} Under assumption (A), regularized continuous Sobolev Descent for $\lambda>0$ is convergent in the MMD sense:
$$\lim_{t\to \infty}  \text{MMD}(\nu_p,\nu_{q_t})=0.$$
\label{pro:conv}
\end{proposition}
\begin{proof}

 \begin{align*}
\frac{1}{2}\frac{d}{dt}\text{MMD}^2(\nu_p,\nu_{q_{t}})&=-(\text{MMD}^2(\nu_p,\nu_{q_{t}})-\lambda \pazocal{S}^2_{\mathcal{H},\lambda}(\nu_p,\nu_{q_{t}}))\\
&\leq 0.
\end{align*}
Since $g(t)=\text{MMD}^2(\nu_p,\nu_{q_{t}})$ is decreasing and positive (bounded from below) it has a finite limit $L$ as $t\to \infty$. 
When $g(t)$ reaches this limit at $t=t_0$ we have $\frac{dg(t)}{dt}|_{t=t_0}=0$, and the graph of $g(t)$ remains constant, $g(t)=g(t_0)=L$ for $t\geq t_0$. Hence $\lim_{t\to \infty} \text{MMD}^2(\nu_p,\nu_{q_{t}})=L=g(t_0)$.
 Under  Assumption (A)  $\frac{dg(t)}{dt}|_{t=t_0}=0$ happens only when $\delta_{p,q_{t_0}}=0$. To see this, note that we have:
$\text{MMD}^2(\nu_p,\nu_{q_{t}})-\lambda \pazocal{S}^2_{\mathcal{H},\lambda}(\nu_p,\nu_{q_{t}})=\lambda \scalT{\delta_{p,q_t}}{(\frac{1}{\lambda}I- (D(\nu_{q_{t}})+\lambda I)^{-1} )\delta_{p,q_{t} }} $. For this term to be zero we have either  (a) $\delta_{p,q_{t}}=0$ or (b) $\delta_{p,q_{t} }\neq 0$ and $\delta_{p,q_{t}}\in Null(D(\nu_{q_t}))$ (See Lemma \ref{lemm:Null} in Appendix \ref{app:Proofs}). The case (b) is excluded by Assumption (A).  Hence, under this assumption, $\frac{dg(t)}{dt}|_{t=t_0}=0$ happens only when $\delta_{p,q_{t_0}}=0$,  i.e when $g(t_0)=\text{MMD}^2(\nu_p,\nu_{q_{t_0}})=\nor{\delta_{p,q_{t_0}}}^2=0$. We conclude therefore that the limit  $L=g(t_0)=0$. 

\end{proof}
\begin{lemma} Let $x\neq 0$, and $D$ a PSD matrix. 
 $\scalT{x}{(\frac{1}{\lambda}I- (D+\lambda I)^{-1} )x }=0$  if and only if $x \in Null (D)$.
 \label{lemm:Null}
\end{lemma}
\begin{proof}
 $H(x)=\scalT{x}{(\frac{1}{\lambda}I- (D+\lambda I)^{-1} )x }=\frac{1}{\lambda}\sum_{j=1}^m\scalT{x}{d_j}^2 -\sum_{j=1}^m  \frac{1}{\lambda_j+\lambda} \scalT{x}{d_j}^2 = \sum_{j=1}^m \left(\frac{1}{\lambda}- \frac{1}{\lambda_j +\lambda}\right) \scalT{x}{d_j}^2$. 
 Let $d_1,\dots d_i$ eigenvectors of $D$ with zero eigenvalues $Null(D)=span\{d_1,\dots d_i\}$.
 Hence:
 $$H(x)=\sum_{j=i+1}^m \left(\frac{1}{\lambda}- \frac{1}{\lambda_j +\lambda}\right) \scalT{x}{d_j}^2 $$
 
($\Leftarrow$) If $x\in Null(D)$ exists $\theta_j$,  such that $x=\sum_{j=1}^{i} \theta_j d_j $, and $x\perp d_{k} $, $k=i+1\dots m$, and hence $H(x)=0$.
 
($\Rightarrow$)  Assume $x\notin Null (D)$ then $x=\sum_{j=i+1}^m \scalT{x}{d_j} d_j$, with $\scalT{x}{d_j}\neq 0, \forall j=i+1\dots m$. Note that  for all  $j=i+1,\dots m$, we have:
$\frac{1}{\lambda}-\frac{1}{\lambda_j+\lambda}>0$ and $\scalT{x}{d_j}^2>0$. Hence $H(x)>0$ for $x\notin Null(D)$. 

\end{proof}

\section{Sobolev Descent with Infinite Dimensional RKHS}\label{app:InfDim}

In this Section we define the Kernelized Sobolev Discrepancy and Descent by looking for the optimal critic in a Hypothesis function class that is a Reproducing Kernel Hilbert Space (RKHS infinite dimensional case). We start first by reviewing some RKHS properties and assumptions needed for our development. 
\subsection{ Kernel Derivative Gramian Embedding of Distributions  }
Let $\mathcal{H}$ be a Reproducing Kernel Hilbert Space with an associated kernel $k:\pazocal{X}\times \pazocal{X}\to \mathbb{R}^{+}$. We make the following assumptions on $\mathcal{H}$:
\begin{enumerate}
\item[A1] There exists $\kappa_1<\infty$ such that $\sup_{x\in \pazocal{X}} \nor{k_x}_{\mathcal{H}}<\kappa_1$.
\item [A2] The kernel is $C^2(\pazocal{X}\times \pazocal{X})$ and there exists $\kappa_2<\infty$ such that for all $a=1\dots d$:\\
$\sup_{x\in \pazocal{X}} Tr((\partial_a k )_x\otimes (\partial_a k )_x )<\kappa_2$.
\item [A3] $\mathcal{H}$ vanishes on the boundary (assuming $\pazocal{X}=\mathbb{R}^d$ it is enough to  have for $f$ in $\mathcal{H}$  $\lim_{\nor{x}\to \infty} f(x)=0$).
\end{enumerate}  
We review here some basic properties of RKHS and function derivatives in RKHS \cite{derivativesRKHS}. The reproducing property give us that $f(x)=\scalT{f}{k_x}_{\mathcal{H}}$ moreover $(D_a f)(x)=\frac{\partial}{\partial x_a}f(x)=\scalT{f}{(\partial_ak)_x}_{\mathcal{H}} $, where $(\partial_a k)_x(t)=\scalT{\frac{\partial k(s,.)}{\partial s_a}\big|_{s=x}}{k_t}$. Note that those two quantities  ($f(x)$ and $(D_a f)(x)$) are well defined and bounded thanks to assumptions A1 and A2. \\
Similar to finite dimensional case we define the Kernel  Derivative Gramian Embedding \textbf{KDGE} of a distribution $\nu_q$ :
\vskip -0.51in
$$D(\nu_q)=  \mathbb{E}_{x\sim \nu_q} \sum_{a=1}^d (\partial_a k)_x\otimes (\partial_a k)_x ~ D(\nu_q)\in \mathcal{H}\otimes \mathcal{H})$$
KDGE is an operator  embedding of the distribution in $\mathcal{H}\otimes \mathcal{H}$, that takes the fingerprint of the distribution with respect to the kernel derivatives averaged over all coordinates. 
The Kernel mean embedding is defined as follows:
$$\mu(\nu_p)=\mathbb{E}_{x\sim \nu_p}k_{x} \in \mathcal{H}.$$
\subsection{Regularized Kernel Sobolev Descent }
Let $\lambda>0$, similarly the Kernel Sobolev Discrepancy has the following form:
$ \pazocal{S}^2_{\mathcal{H},\lambda}(\nu_p,\nu_q)=\nor{(D(\nu_q)+\lambda I)^{-\frac{1}{2}} (\mu(\nu_p)-\mu(\nu_q))}^2_{\mathcal{H}},$
and the Sobolev critic is defined as follows:
$u^{\lambda}_{p,q}=(D(\nu_q)+\lambda I)^{-1}(\mu(\nu_p)-\mu(\nu_q)) \in \mathcal{H}$
its evaluation function is 
$u^{\lambda}_{p,q}(x)=\scalT{(D(\nu_q)+\lambda I)^{-1}(\mu(\nu_p)-\mu(\nu_q))}{k_x}_{\mathcal{H}}$ and it derivatives for $a=1\dots d$:
$\partial_{a} u^{\lambda}_{p,q}(x)=\scalT{(D(\nu_q)+\lambda I)^{-1}(\mu(\nu_p)-\mu(\nu_q))}{\partial_{a}k_x}_{\mathcal{H}}$.

The following Theorem for inf. Dim RKHS parallels Theorem \ref{theo:PertubationMMDSobolev} for finite Dim RKHS. Hence inf. Dim Kernel Sobolev descent decreases the MMD as well, and as discussed in Section \ref{sec:convergence} under Assumption (A) using a characteristic or a universal kernel we garantee the convergence of Sobolev descent in the MMD sense as well as in distribution. 

Note that in the case $\lambda=0$, Sobolev critic is not well defined unless we assume that $\mu(\nu_p)-\mu(\nu_q)$ is in the range of $D(\nu_q)$. If we make this assumption the following theorem  holds also for $\lambda=0$.

\begin{theorem}[Transport Using Gradient flows of Infinite dim. RKHS]Let $\lambda>0$ . Let $u^{\lambda}_{p,q}$ the unnormalized  solution of the regularized Kernel Sobolev discrepancy between $\nu_p$ and $\nu_q$ i.e $u^{\lambda}_{p,q}=(D(\nu_q)+\lambda I)^{-1}(\mu(\nu_p)-\mu(\nu_q)).$ Consider $d\chi_{u}(x)=-div(q(x)\nabla_x u^{\lambda}_{p,q}(x))dx$, i.e corresponding to the infinitesimal transport of $\nu_q$ via $T(x)=x+\varepsilon \nabla_x u^{\lambda}_{p,q}(x)$. 
We have the following first variation of the $\text{MMD}^2$ under this particular perturbation:
$$\frac{d}{d\varepsilon}\text{MMD}^2(\nu_p,\nu_q+\varepsilon\chi_{u})\Big|_{\varepsilon=0}=- 2\left(\text{MMD}^2(\nu_p,\nu_q)-\lambda \pazocal{S}^2_{\mathcal{H},\lambda}(\nu_p,\nu_q)\right)\leq 0. $$
\label{theo:PertubationMMDSobolevINFDim}
\end{theorem}
\begin{proof}[Proof of Theorem \ref{theo:PertubationMMDSobolevINFDim}]
\begin{align*}
\frac{1}{2}\frac{d}{d\varepsilon}\text{MMD}^2(\nu_p,\nu_q+\varepsilon\chi_u)\Big|_{\varepsilon=0}&=-\int \delta_{p,q}d\chi_u=-\int_{\pazocal{X}}\delta_{p,q}(x)(-div(q(x)\nabla_x u^{\lambda}_{p,q}(x)))dx \\
&=\int_{\pazocal{X}}\delta_{p,q}(x)div(q(x)\nabla_x u^{\lambda}_{p,q}(x))dx\\
&=-\int_{\pazocal{X}} \scalT{\nabla_x \delta_{p,q}(x)}{\nabla_x u^{\lambda}_{p,q}(x)}q(x)dx \text{ (Divergence theorem and zero boundary)}\\
&=-\int_{\pazocal{X}}\sum_{a=1}^d \scalT{\delta_{p,q}}{\partial_a k_x}_{\mathcal{H}} \scalT{u^{\lambda}_{p,q}}{\partial_a k_x}_{\mathcal{H}}q(x)dx\\
&=-\int_{\pazocal{X}}\scalT{\delta_{p,q}}{\left(\sum_{a=1}^d \partial_a k_x \otimes \partial_a k_x\right) u^{\lambda}_{p,q}}_{\mathcal{H}} q(x) dx\\
&=-\scalT{\delta_{p,q}}{\left(\int_{\pazocal{X}}\left(\sum_{a=1}^d \partial_a k_x \otimes \partial_a k_x\right)q(x) dx\right) u^{\lambda}_{p,q} }_{\mathcal{H}}\\
&=-\scalT{\delta_{p,q}}{\left(\mathbb{E}_{x\sim \nu_q}\sum_{a=1}^d \partial_a k_x \otimes \partial_a k_x\right)u^{\lambda}_{p,q}}_{\mathcal{H}}\\
&=- \scalT{\delta_{p,q}}{D(\nu_q)u^{\lambda}_{p,q}}_{\mathcal{H}} \text{(by definition)}\\
&=- \scalT{\delta_{p,q}}{\left(D(\nu_q)+ \lambda I -\lambda I\right) u^{\lambda}_{p,q}}_{\mathcal{H}}\\
& =- \scalT{\delta_{p,q}}{\left(D(\nu_q)+ \lambda I\right) u^{\lambda}_{p,q}}_{\mathcal{H}}+\lambda \scalT{\delta_{p,q}}{u^{\lambda}_{p,q}}_{\mathcal{H}}\\
\end{align*}
Recall that : $$\left(D(\nu_q)+ \lambda I\right) u^{\lambda}_{p,q}=\delta_{p,q},$$
and by definition the regularized Kernel Sobolev Discrepancy we have:
$$\scalT{\delta_{p,q}}{u^{\lambda}_{p,q}}_{\mathcal{H}}=\pazocal{S}^2_{\mathcal{H},\lambda}(\nu_p,\nu_q),$$
Hence replacing the expressions above we obtain:
\begin{align*}
\frac{1}{2}\frac{d}{d\varepsilon}\text{MMD}^2(\nu_p,\nu_q+\varepsilon\chi)\Big|_{\varepsilon=0}&= - \scalT{\delta_{p,q}}{\delta_{p,q}}_{\mathcal{H}}+ \lambda \pazocal{S}^2_{\mathcal{H},\lambda}(\nu_p,\nu_q)\\
&=-\text{MMD}^2(\nu_p,\nu_q)+ \lambda \pazocal{S}^2_{\mathcal{H},\lambda}(\nu_p,\nu_q)\\
&= -\left(\text{MMD}^2(\nu_p,\nu_q)-\lambda \pazocal{S}^2_{\mathcal{H},\lambda}(\nu_p,\nu_q)\right)
\end{align*}

 Hence:
\begin{align*}
\pazocal{S}^2_{\mathcal{H},\lambda}(\nu_p,\nu_q) = \scalT{\delta_{p,q}}{(D(\nu_q)+\lambda I)^{-1}\delta_{p,q}}&\leq \nor{(D(\nu_q)+\lambda I)^{-1}}_{\mathcal{L}(\mathcal{H})}\nor{\delta_{p,q}}^2_{\mathcal{H}}\\
&\leq \frac{1}{\lambda}\text{MMD}^2(\nu_p,\nu_q),
\end{align*}
It follows that : $$\text{MMD}^2(\nu_p,\nu_q)-\lambda \pazocal{S}^2_{\mathcal{H},\lambda}(\nu_p,\nu_q) \geq 0$$
and 
$$\frac{1}{2}\frac{d}{d\varepsilon}\text{MMD}^2(\nu_p,\nu_q+\varepsilon\chi)\Big|_{\varepsilon=0}= -\left(\text{MMD}^2(\nu_p,\nu_q)-\lambda \pazocal{S}^2_{\mathcal{H},\lambda}(\nu_p,\nu_q)\right)\leq 0.$$
\end{proof}

\section{Continuous Regularized Kernel Sobolev Descent}
This section gives some more intuition on a continuous form of Sobolev Descent.
\paragraph{Non linear Fokker Planck and Deterministic Mckean Vlasov Processes.}

The regularized Kernel Sobolev descent can be seen as a continuous process, written in this primal form:  
$$\min_{u_{p,q_t} \in \mathcal{H},q_t} \int_{0}^{\infty} \left( \int_{\pazocal{X}}\nor{\nabla_xu_{p,q_t}(x)}^2 d\nu_{q_t}(x)+\lambda \nor{u_{p,q_t}}^2_{\mathcal{H}}-2(\mathbb{E}_{x\sim p} u_{p,q_t}(x)-\mathbb{E}_{x\sim \nu_{q_t}}u_{p,q_t}(x)) \right)dt$$
$$  \frac{\partial q_t}{\partial t}(x)=-div(q_{t}(x) \nabla_x u_{p,q_t}(x)), \nu_{q_0}=\nu_{q}$$
This form gives us the interpretation that we are seeking potentials $u_{p,q_t}$ in the finite dimensional RKHS, that have minimum regularized kinetic energy $ \int_{\pazocal{X}}\nor{\nabla_xu_{p,q_t}(x)}^2 d\nu_{q_t}(x)+\lambda \nor{u_{p,q_t}}^2_{\mathcal{H}}$  and that advects $q_t$ to $p$. The advection  can be seen informally by noting that we want to maximize $\mathbb{E}_{x\sim p} u_{p,q_t}(x)-\mathbb{E}_{x\sim \nu_{q_t}}u_{p,q_t}(x))=\scalT{\boldsymbol{u}_{p,q_t}}{\boldsymbol{\mu}(p)-\boldsymbol{\mu}(q_{t})}$, meaning we want $\boldsymbol{u}_{p,q_t}$ to be aligned with the correct transport direction from $q$ to $p$ . The evolution of the density is then dictated by the non linear fokker planck equation known as the deterministic Mckean-Vlasov equation:
$$\frac{\partial q_t}{\partial t}(x)=-div(q_{t}(x) \nabla_x u_{p,q_t}(x))$$

The primal form given above is not computational friendly and hence we are using 1)  the dual form of the Sobolev Discrepancy and 2) the equivalence between stochastic differential equation in general and the Mckean Vlasov process, as summarized below:

$$\sup_{f_{p,q_t} \in \mathcal{H},q_t} \int_{0}^{\infty} (\mathbb{E}_{x\sim p}f_{p,q_t}(x)-\mathbb{E}_{x\sim q_t}f_{p,q_t}(x)) dt $$
$$ \text{s.t }  \mathbb{E}_{x\sim q_t}\nor{\nabla_x f_{p,q_t}(x)}^2+\lambda \nor{f_{p,q_t}}^2_{\mathcal{H}}\leq 1$$
$$u_{p,q_t}=\pazocal{S}_{\mathcal{H},\lambda}(\nu_p,\nu_{q_t}) f^*_{p,q_t}$$
 $$ dX_t =\nabla_x u_{p,q_t}(X_t) dt ~ X_t \sim \nu_{q_t}~ X_0 \sim \nu_q$$
Finite dimensional RKHS Sobolev descent is exploiting this computational friendly formulation: $u_{p,q_t}$ has a closed form solutions at each time $t$.
Neural Sobolev Descent is also using this formulation by solving the optimization problem for each $u_{p,q_t}$ using gradient descent and an augmented lagrangian.
\paragraph{What happens when considering $\mathcal{H}=W^{1,2}_{0}$ and no Regularization?}

\begin{theorem}[Convergence of the continuous limit of Sobolev Descent]
Consider particles  $X_0$ with density function $q_0=q$ the source density. Let $\nu_p$ be the target measure whose density is $p$.
Consider the following continuous process:
\begin{equation}
dX_t=  \pazocal{S}(\nu_p,\nu_{q_t})\nabla_x f^*_{\nu_p,\nu_{q_t}}(x) dt,
 \label{eq:continuousDescent}
\end{equation}
let $q_t$ be the density function of particles $X_t$ and $f^*_{\nu_p,\nu_{q_t}}$ the optimal Sobolev critic between $\nu_p$ and $\nu_{q_t}$(whose densities are $p$ and $q_t$ respectively).
We have:
$$q_t(x)= \left(1-e^{-t}\right)p(x) + e^{- t}q(x),$$
 The density $q_t$ of the particles $X_t$ approaches the target density $p$, as $t\to \infty$ ( therefore as $t\to \infty$  $q_t \to p$).
 \label{theo:ContinuousDescentW2}
\end{theorem}
We see therefore that the unregularized theoretical Sobolev descent boils down also to interpolation, hence the crucial role of regularization.
\begin{proof}[Proof of Theorem \ref{theo:ContinuousDescentW2}]
Let  $f^*_{\nu_p,\nu_{q_t}}$ be the Sobolev critic between $q_t$ and $p$, it satisfies the following PDE (See \cite{SobolevGAN} for instance) :
\begin{equation}
 p(x)-q_t(x)=-\pazocal{S}(\nu_p,\nu_{q_t}) \text{div} (q_t(x)\nabla_x f^*_{\nu_p,\nu_{q_t}}(x)),
 \label{eq:fokkerPlanckFormSobolev}
\end{equation}
where $q_t$ is the distribution of the particles moving with the flow: 
$$dX_t = \pazocal{S}(\nu_{p},\nu_{q_t}) \nabla_x f^*_{\nu_p,\nu_{q_t}}(X_t) dt, \text{where the density of } X_0 \text{ is given by } q_0(x)=q(x)$$
by non linear fokker planck equation and results on Mckean Vlasov processes~\cite{funaki1984certain}, the distribution $q_t$ evolves as follows:
\begin{equation}
\frac{\partial }{\partial t} q_t(x)= - \pazocal{S}(\nu_p,\nu_{q_t}) \text{div} (q_t(x)\nabla_x f^*_{\nu_p,\nu_{q_t}}(x))
\label{eq:FokkerPlanckSobolev}
\end{equation}
From Equation \eqref{eq:fokkerPlanckFormSobolev} and \eqref{eq:FokkerPlanckSobolev} we see that:
$$\frac{\partial}{\partial t}q_t(x)=  \left(p(x)-q_t(x)\right),$$
in other words:
$$\frac{\partial }{\partial t}(p(x)-q_t(x))= -\left(p(x)-q_t(x)\right),$$
Hence :
\vskip -0.3 in
\begin{eqnarray*}
p(x)-q_t(x)&=& \left(p(x)-q_0(x)\right)e^{- t}\\
&=& e^{- t}\left(p(x)-q(x)\right)
\end{eqnarray*}
\noindent It follows:
$$q_t(x)= \left(1-e^{- t}\right)	p(x) + e^{- t}\underbrace{q(x)}_{q_0(x)}$$
therefore as $t\to \infty$,  $q_t \to p$.
\end{proof}

\section{Regularization as smoothing of Principal Transport Directions.} 
\label{app:regu_smoothing}
In order to further understand the role of regularization let us take a close look on the expression of the Sobolev critic.
Let $(\lambda_j,\boldsymbol{d_j}),j=1\dots m$ be Eigen system  the KDGE $D(\nu_q)$.
We have: $\boldsymbol{u}^{\lambda}_{p,q}=(D(\nu_q)+\lambda I)^{-1}(\boldsymbol{\mu}(\nu_p)-\boldsymbol{\mu}(\nu_q))=\sum_{j=1}^{m}\frac{1}{\lambda_j+\lambda}\scalT{\boldsymbol{d_j}}{\boldsymbol{\mu}(\nu_p)-\boldsymbol{\mu}(\nu_q)}d_j.$
It follows that $ \nabla_x u^{\lambda}_{p,q}(x)=\sum_{j=1}^{m}\frac{1}{\lambda_j+\lambda}\scalT{\boldsymbol{d_j}}{\boldsymbol{\mu}(\nu_p)-\boldsymbol{\mu}(\nu_q)}\nabla_x d_j(x),$
where $\nabla_x d_j(x)=[J\Phi(x)]\boldsymbol{d_j}$ and $J\Phi(x) \in \mathbb{R}^{d\times m}$ is the jacobian of $\Phi$, $[J\Phi]_{a,j}(x)=\frac{\partial}{\partial x_a}\Phi_j(x) $.
One can think of $\nabla_x d_j(x)$ as \emph{principal transport directions.}
Regularization is introducing therefore a spectral filter on the principal transport directions by weighing down directions with low eigenvalues $\frac{1}{\lambda_j+\lambda}\approx \frac{1}{\lambda}$ for $\lambda_j<\lambda$, and $\frac{1}{\lambda_j+\lambda}\approx\frac{1}{\lambda_j}$ otherwise.
Principal transport directions with small eigenvalues contribute to the fast exponential convergence and may result in discontinuous paths.
Regularization filters out those directions, resulting in smoother probability paths between $\nu_q$ and $\nu_p$.
\begin{figure*}[ht!]
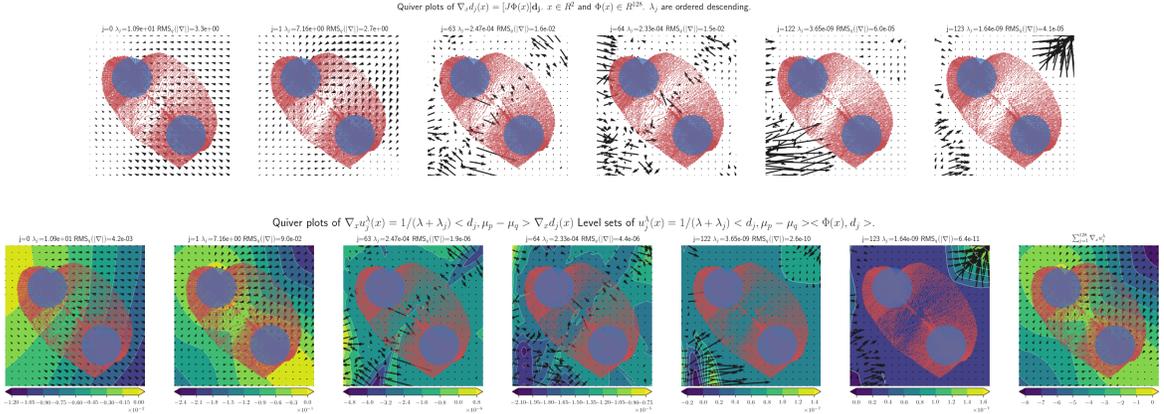

\centering
\includegraphics[height=80px]{figs/principdirs_neural_dj.png}
\includegraphics[height=80px]{figs/principdirs_neural_uj.png}
\vskip -0.2in
\caption{
The principal transport directions for an intermediate state $q_t$ (red cloud) in the shape morphing application with Neural Sobolev Descent (see Figure \ref{fig:morphing_nsd}).
The top row shows $\nabla_x d_j(x)$, bottom row shows $\nabla_x u_j^\lambda(x)$ for $\lambda=0.3$.
Note how small $j$ (large eigenvalues) correspond to smooth vectorfields where the vectors have large norm (as measured in RMS over the points in point cloud $x \sim \nu_{q_t}$).
The intermediate and large $j$ values correspond to non-smooth vectorfields and non-smooth motions. For $\nabla_x u_j^\lambda(x)$, the principal transport directions $\nabla_x d_j(x)$ are multiplied with $\frac{1}{\lambda+\lambda_j}$ and the inner product with $\mu_p-\mu_q$ (two scalar multipliers). We see the non-smooth $\nabla_x u_j^\lambda(x)$ (small $\lambda_j$) have small RMS norm and contribute less, as they are effectively filtered out by the smoothing parameter $\lambda$.
The bottom right subplot shows the total critic $u^\lambda(x)=\sum_{j=1}^m u^\lambda_j(x)$.
}
\label{fig:Transpreput}
\vskip -0.2in
\end{figure*}
\section{Relation to Previous works} 
\label{app:diagram_ot}
\begin{figure}[ht!]
\centering
\includegraphics[width=0.3\textwidth]{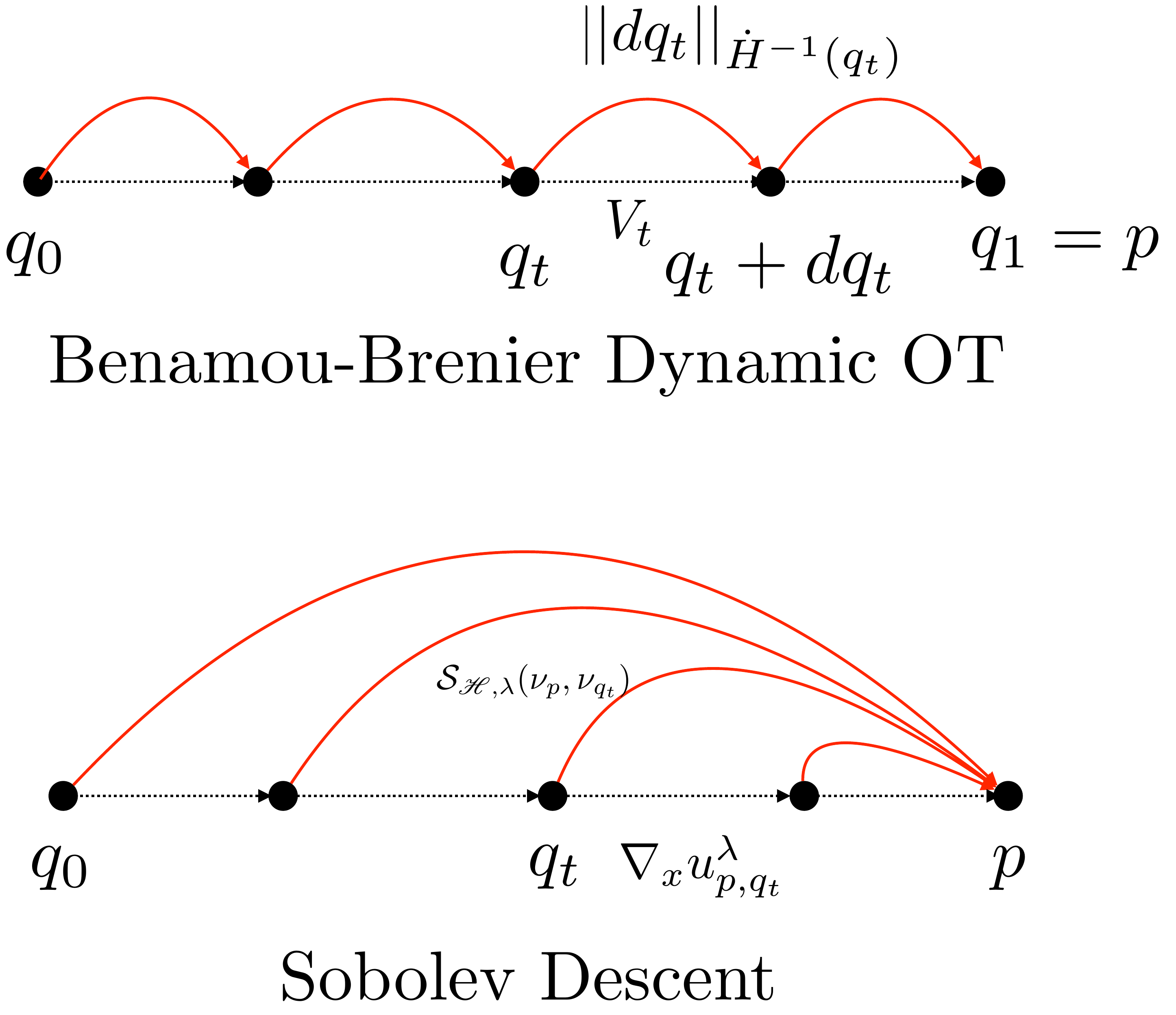}
\caption{Both formulations minimize a form of kinetic energy, represented with red arrows.
  While this energy is between consecutive timesteps for dynamic OT (Benamou-Brenier),
  it is between the current distribution and the target for Sobolev Descent.
  The velocity fields are represented with dashed arrows, and can be expressed through gradients of a convex potential for dynamic OT.
  For Sobolev Descent, the velocity fields are the gradient of the Sobolev Critic.
}
\label{fig:diagram_ot}
\end{figure}

\begin{table}[ht!] 
\centering
\resizebox{\textwidth}{!}{
\begin{tabular}{| l | l | l | l | l | }
\hline
 &~Densities~~~ &~~~~~~~Kinetic Energy (KE)  &~~~~~~~~~~~~Paths  \\
 & &~~~~~~~~~~and Velocities~~~ &~~~~~~~~~ \\
\hline
Benamou   &$p,q$ & KE between $q_t$ and $q_{t}+dq_{t}$  & Optimal Paths \\
Brenier  &known &Min KE$=\nor{dq_t}^2_{\dot{H}^{-1}(q_t)}$  &minimizing KE\\
\cite{dynamicTransport}    & &Min KE$=\int_{\pazocal{X}}\nor{V^*_{t}(x)}^2q_{t}(x)dx$  &between time steps \\
& & Velocity $V^*_{t}$ from critic of $\nor{dq_t}^2_{\dot{H}^{-1}(q_t)}$   &$T_{\#}(\nu_{q})=\nu_p, $\\
& &  &$W^2_2(p,q)=\int_{0}^1\nor{dq_t}^2_{\dot{H}^{-1}(q_t)}dt$\\
\hline
\hline
Stein &  $p$ known&KE between $q_t$ and $p$    &Paths \\
Descent&  samples   &Velocity $\varphi^*_{p,q_t}(x)$ &minimizing KL divergence\\
\cite{steindescent,Stein}&  $\sim q$ & $\varphi^*_{p,q_t} \in \mathcal{H}^d$ critic of $\mathbb{S}^2(p,q_{t})$    &between $q_t$ and target $p$\\
&   & $\text{KE}= \int_{\pazocal{X}}\nor{\varphi^*_{p,q_t}(x)}^2q_t(x)dx$ (not min) &$\lim_{t\to\infty} \text{KL}(q_t,p)= 0$  \\
 \hline 
\hline
Reg. & samples   &KE between $q_t$ and $p$  & Tunable paths via $\lambda$\\
Sobolev   &   $\sim p$&Min Reg KE$=\pazocal{S}^2_{\mathcal{H},\lambda}(p,q_{t})$  &minimizing Reg. KE \\
 Descent& samples& $=\int_{\pazocal{X}} \nor{\nabla_xu^{\lambda}_{p,q}}^2q_{t}(x)dx+\lambda \nor{u^{\lambda}_{p,q_t}}^2_{\mathcal{H}}$  &between $q_t$ and target $p$  \\
(This work)&$\sim q$ & Velocity $\mathbf{\nabla_x} u^{\lambda}_{p,q_t}(x)$  &$\lim_{t\to\infty} \text{MMD}(q_t,p)= 0$\\
& &  $ u^{\lambda}_{p,q_t} \in \mathcal{H}$ critic of $\pazocal{S}^2_{\mathcal{H},\lambda}(p,q_{t})$  & \\
  \hline 
\end{tabular}}
\caption{Comparison with Benamou-Brenier and Stein Descent.}
\label{tab:comparison}
\vskip -0.2in
\end{table}

\section{Algorithm}
\begin{algorithm}[ht!]
\caption{Empirical Kernelized Sobolev Descent}
\begin{algorithmic}
 \STATE {\bfseries Inputs:} $\varepsilon$ Learning rate, $L$ number of iterations \\
  $\{x_i,i=1\dots N\}$, drawn from target distribution $\nu_p$,
   $\{y_j,j=1\dots M\}$ drawn from source distribution $\nu_q$
   $\mathcal{H}$ a Hypothesis Class \\
 \STATE {\bfseries Initialize} $x^0_j=y_j, j=1\dots M$ 
 \FOR{ $\ell=1\dots L$}
\STATE \emph{\bfseries Critic Update}\\
 \STATE Compute Sobolev Critic in $\mathcal{H}$, between $q_{\ell-1}$ and $p$\\
 \STATE $\boldsymbol{\hat{u}}^{\lambda}_{p,q_{\ell-1}}= \left(\hat{D}(\hat{\nu}_{q_{\ell-1}})+\lambda I_m\right)^{-1} \left(\hat{\boldsymbol{\mu}}(\hat{\nu}_p)-\hat{\boldsymbol{\mu}}(\hat{\nu}_{q_{\ell-1}}) \right)$
 \STATE \emph{\bfseries Particles Update}\\
 \FOR{$j=1$ {\bfseries to} $M$}
 \STATE $x^{\ell}_j = x^{\ell-1}_j +\varepsilon \nabla_x \hat{u}^{\lambda}_{p,q_{\ell-1}}(x^{\ell-1}_j) $ ($q_{\ell}$ is the density of the particles $x^{\ell}_j$)
 \ENDFOR
 \ENDFOR
 \STATE {\bfseries Output:} $\{x^L_j, j=1\dots M\}$ 
 \end{algorithmic}
 \label{alg:KSD}
\end{algorithm}

\begin{algorithm}[ht!]
\caption{Neural Sobolev Descent (ALM Algorithm)}
\begin{algorithmic}
 \STATE {\bfseries Inputs:} $\varepsilon$ Learning rate particles, $n_c$ number of critics updates, $L$ number of iterations \\
  $\{x_i,i=1\dots N\}$, drawn from target distribution $\nu_p$\\
   $\{y_j,j=1\dots M\}$ drawn from source distribution $\nu_q$\\
   Neural critic $f_{\xi}(x)=\scal{v}{\Phi_{\omega}(x)}$, $\xi=(v,\omega)$ parameters of the neural network\\
 \STATE {\bfseries Initialize} $x^0_j=y_j, j=1\dots M$ 
 \FOR{ $\ell=1\dots L$}
\STATE \emph{\bfseries Critic Update}\\
\STATE (between particles updates gradient descent on  the critic is initialized from previous episodes)
 \FOR{$j=1$ {\bfseries to} $n_c$}
 \STATE  $\hat{\mathcal{E}}(\xi)\gets \frac{1}{N}\sum_{i=1}^N f_{\xi}(x_i) -\frac{1}{M}\sum_{j=1}^M f_{\xi}(x^{\ell-1}_j)$
 \STATE $\hat{\Omega}(\xi)\gets \frac{1}{M}\sum_j \nor{\nabla_x f_{\xi}(x^{\ell-1}_j)}^2 $
 \STATE $\pazocal{L}_{S}(\xi,\lambda)= \hat{\mathcal{E}}(\xi)+ \lambda(1-\hat{\Omega}(\xi))-\frac{\rho}{2}(\hat{\Omega}(\xi)-1)^2$
  \STATE $(g_{\xi},g_{\lambda})\gets (\nabla_{\xi} {\pazocal{L}_{S}},\nabla_{\lambda}\pazocal{L}_{S})(\xi,\lambda) $
 \STATE $ \xi\gets \xi +\eta \text{ ADAM }(\xi,g_{\xi})$\\
 \STATE $\lambda \gets \lambda - \rho g_{\lambda}$ \COMMENT{SGD rule on $\lambda$ with learning rate $\rho$}
 \ENDFOR
 \STATE \emph{\bfseries Particles Update}\\
 \FOR{$j=1$ {\bfseries to} $M$}
 \STATE $x^{\ell}_j = x^{\ell-1}_j +\varepsilon \nabla_x f_{\xi}(x^{\ell-1}_j ) $ (current $f_{\xi}$ is the critic between $q_{\ell-1}$ and $p$ )
 \ENDFOR
 \ENDFOR
 \STATE {\bfseries Output:} $\{x^L_j, j=1\dots M\}$ 
 \end{algorithmic}
 \label{alg:NSD}
\end{algorithm}


%

\section{Additional Figures of Sobolev Descent for Image color transfer and shape morphing}
\label{app:exp_additional}
\subsection{Color Transfer}
See Figure \ref{fig:coloring_1_sigma} and Figure \ref{fig:coloring_1_timetrace}.
\begin{figure}[h!]
\centering
\includegraphics[width=0.5\textwidth]{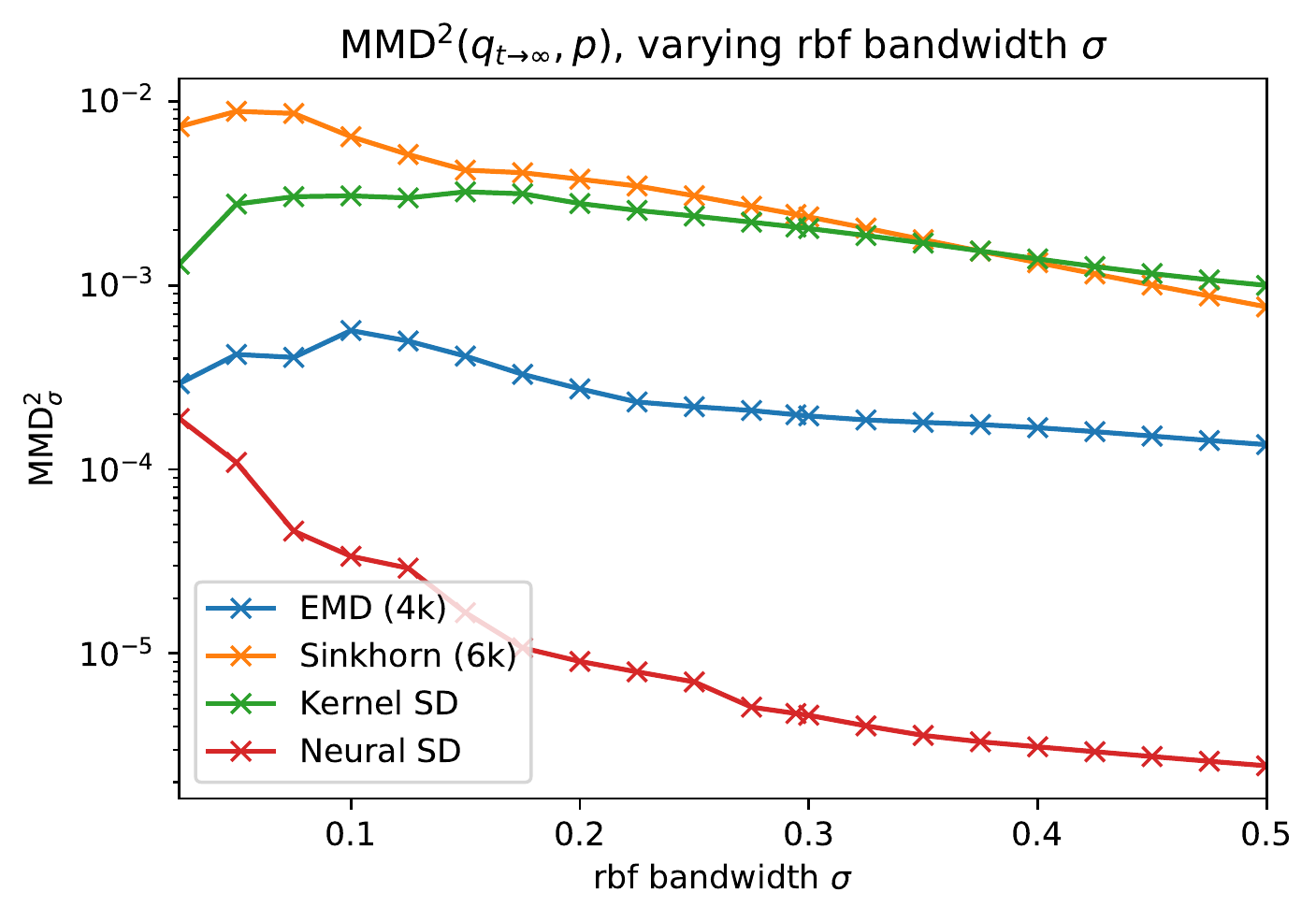}
\caption{Coloring MMD across a range of rbf bandwidths, using the final $q_t$ from Figure \ref{fig:coloring}.
We select $\sigma=0.1$ for the main Figure \ref{fig:coloring}.}
\label{fig:coloring_1_sigma}
\end{figure}

\begin{figure}[h!]
\centering
\includegraphics[width=\textwidth]{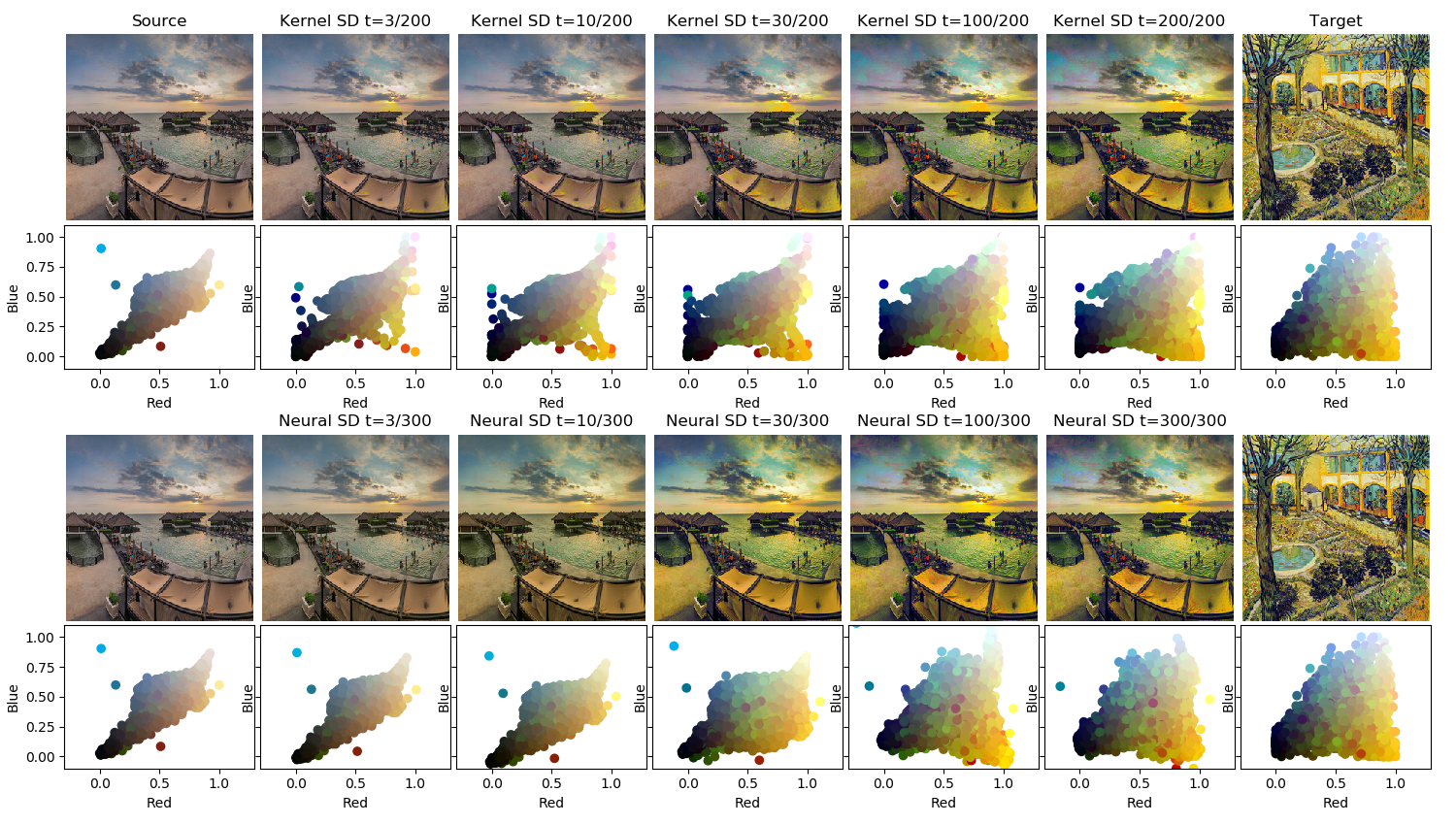}
\caption{
  Evolution of $q_t$ for Kernel and Neural Sobolev Descent.
}
\label{fig:coloring_1_timetrace}
\end{figure}

\subsection{Shape morphing: Convergence speed}\label{app:conv}
Figure \ref{fig:convergence} shows MMD convergence for shape morphing with Kernel Sobolev Descent.
 \begin{figure}[h!]
  \centering
  \includegraphics[width=0.8\textwidth]{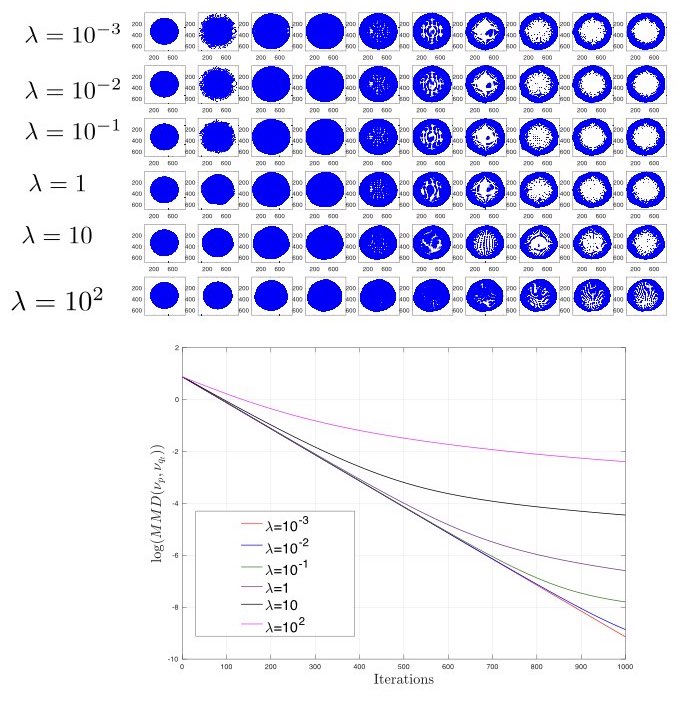}
  \caption{Shape morphing with Kernel Sobolev Descent:
    We see in this figure that for small regularization the convergence is exponential (linear in log MMD scale).
    For higher lambda values, regularization is slows down the convergence and smooths out the trajectories from $q$ to $p$.
    We see that for small lambda high frequency motions appearing in early time steps.
    Those high frequency trajectories are smoothed out with higher regularization, confirming what our theory predicts, on the effect of regularization as a spectral filtering of principal transport direction of the KDGE, favoring smoother distribution paths.}
\label{fig:convergence}
\end{figure}

\subsection{Shape morphing: Neural Sobolev Descent Level sets and Quiver plots}
\label{app:exp_morphing}
\noindent \textbf{Implementation details.}
We scaled the input coordinates to be in the $[-1,1]$ range. The neural network, implemented in pytorch,
is a simple multi-layer perceptron (MLP) with 3 hidden layers (32, 64, 32 respectively), input size 2 and output size 1 ($=f_\xi(x) \in \mathbb{R}$),
and Leaky ReLU nonlinearities with negative slope 0.2.
We use adam with learning rate $\eta=5 \text{e}^{-4}$ for $f_\xi$ and $\varepsilon=3 \text{e}^{-3}$.
For penalty weights we have $\rho=1\text{e}^{-6}$ and initialize with $\lambda=0.01$.
We use $n_c=10$ (for the first time step we warm up with $n_c=50$), and run the descent for $T=800$ steps.
Code is available on \url{https://goo.gl/tncxQm}. Videos of shapes morphing are available on \url{https://goo.gl/X4o8v6}.
\begin{figure}[h!]
  \centering
  \includegraphics[width=0.88\textwidth]{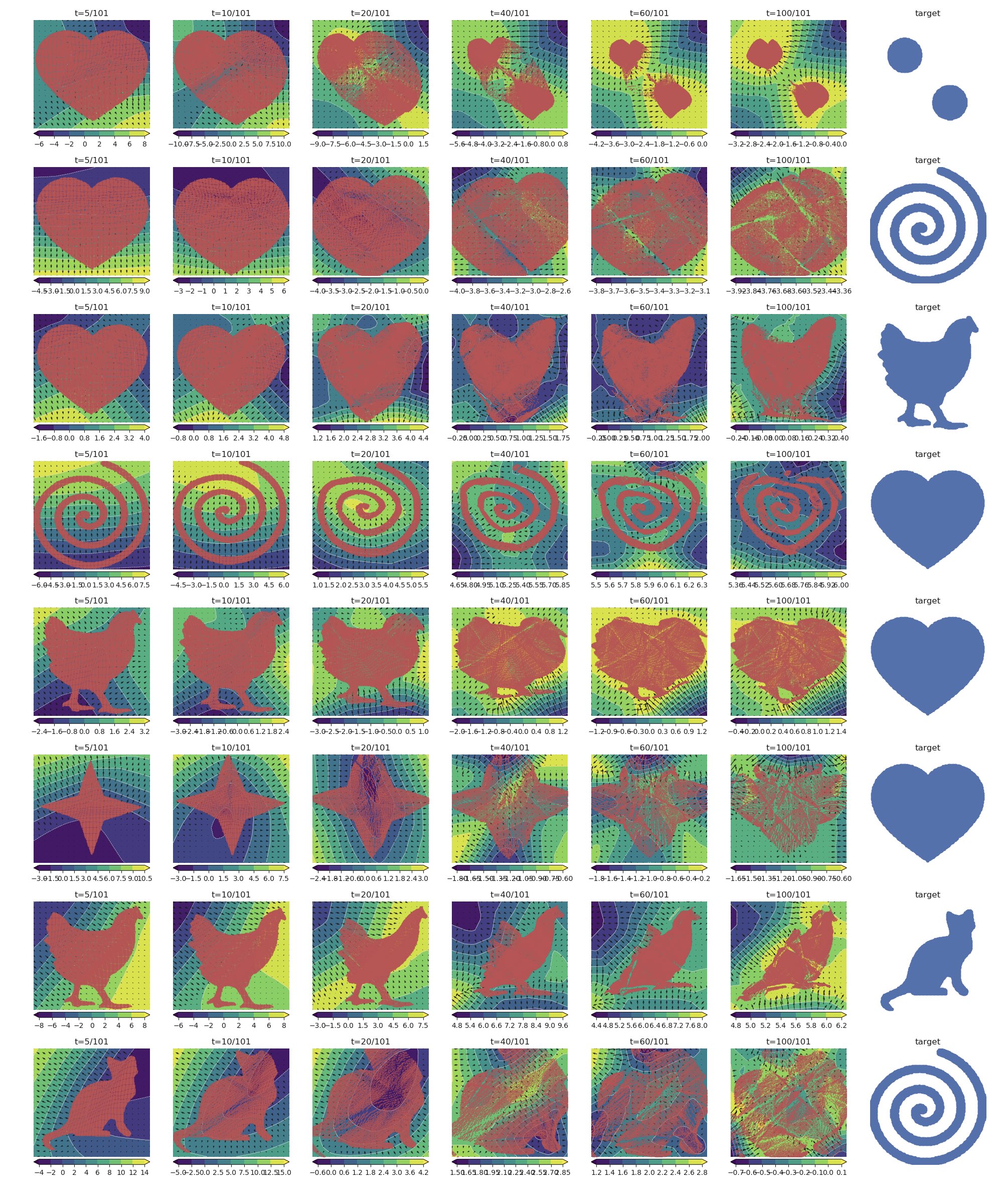}
  \caption{Level sets of $f_\xi(x)$ and quiver plots showing $\nabla_x f_\xi(x)$ for the first 100 timesteps of the Neural Sobolev Descent shape morphing results from Figure \ref{fig:morphing_nsd}.
  Videos are available on \url{https://goo.gl/X4o8v6}.
}
\label{fig:morphing_nsd_quivers}
\end{figure}

\begin{figure}[ht!]
\centering
\includegraphics[width=0.5\textwidth]{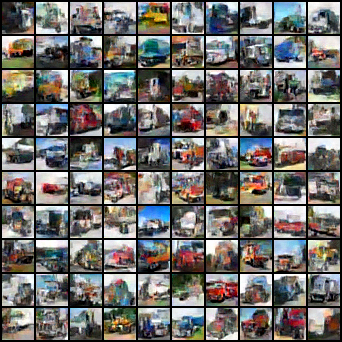}
\captionof{figure}{Particles (Images) of Neural Sobolev Descent at convergence, when the target distribution is the trucks class of CIFAR 10 and the Sobolev critic is a learned CNN. }
\label{fig:sampleDescentbis}
\end{figure}

\begin{figure}[ht!]
\centering
\includegraphics[width=0.5\textwidth]{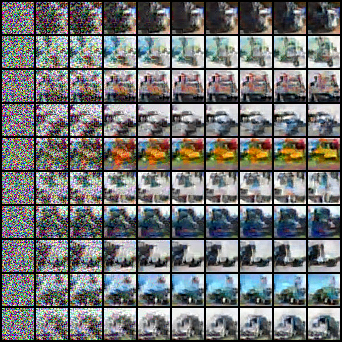}
\captionof{figure}{Paths of Particles (Images) in Neural Sobolev Descent from noise to CIFAR truck classes }
\label{fig:sampleDescentbis}
\end{figure}

\end{document}